\documentclass{article}

\usepackage{etoolbox}
\newtoggle{neurips}
\togglefalse{neurips}
\newcommand{\neurips}[1]{\iftoggle{neurips}{#1}{}}
\newcommand{\arxiv}[1]{\iftoggle{neurips}{}{#1}}
\newcommand{\loose}{\looseness=-1}

\neurips{
  \PassOptionsToPackage{numbers, compress, square}{natbib}
  \usepackage[]{neurips_2024}
  \usepackage[numbers]{natbib}
  \bibliographystyle{abbrvnat}
}

\usepackage[utf8]{inputenc} %
\usepackage[T1]{fontenc}    %
\usepackage{url}            %
\usepackage{booktabs}       %
\usepackage{amsfonts}       %
\usepackage{nicefrac}       %
\usepackage{microtype}      %

\usepackage{tocloft}            %

\usepackage[shortlabels]{enumitem}
\neurips{
\setlist[enumerate]{leftmargin=*}
\setlist[itemize]{leftmargin=*}
}

\usepackage{breakcites}
\usepackage[normalem]{ulem}

\usepackage{mathrsfs}

\usepackage{algorithm}
\usepackage{verbatim}
\usepackage[noend]{algpseudocode}
\newcommand{\multiline}[1]{\parbox[t]{\dimexpr\linewidth-\algorithmicindent}{#1}}

\usepackage{multicol}

\usepackage{colortbl}

\usepackage{setspace}

\usepackage{transparent}

\usepackage{inconsolata}
\usepackage[scaled=.90]{helvet}
\usepackage{xspace}

\usepackage{pifont}
\usepackage{bm}

\DeclareFontFamily{U}{jkpmia}{}
\DeclareFontShape{U}{jkpmia}{m}{it}{<->s*jkpmia}{}
\DeclareFontShape{U}{jkpmia}{bx}{it}{<->s*jkpbmia}{}
\DeclareMathAlphabet{\mathfrak}{U}{jkpmia}{m}{it}
\SetMathAlphabet{\mathfrak}{bold}{U}{jkpmia}{bx}{it}

\usepackage{./txie}

\arxiv{
\usepackage[letterpaper, left=1in, right=1in, top=1in,
bottom=1in]{geometry}
  \usepackage{parskip}
}

\neurips{
  \usepackage{parskip}
  }

\PassOptionsToPackage{hypertexnames=false}{hyperref}  %

\usepackage[svgnames]{xcolor}
\usepackage[colorlinks=true, linkcolor=blue!70!black, citecolor=blue!70!black,urlcolor=black,breaklinks=true]{hyperref}
\colorlet{txblue}{RoyalBlue!70!NavyBlue}
\hypersetup{linkcolor=txblue,
            citecolor=txblue}
\usepackage{microtype}
\usepackage{hhline}

\usepackage{amsthm}
\usepackage{mathtools}
\usepackage{amsmath}
\usepackage{bbm}
\usepackage{amsfonts}
\usepackage{amssymb}
\usepackage[nameinlink,capitalize]{cleveref}

\makeatletter
\newcommand{\neutralize}[1]{\expandafter\let\csname c@#1\endcsname\count@}
\makeatother

\newenvironment{thmmod}[2]
  {\renewcommand{\thetheorem}{\ref*{#1}#2}%
   \neutralize{theorem}\phantomsection
   \begin{theorem}}
  {\end{theorem}}

\usepackage{algorithm}

\arxiv{
\usepackage{natbib}
\bibliographystyle{plainnat}
\bibpunct{(}{)}{;}{a}{,}{,}
}

\usepackage{xpatch}

\usepackage{thmtools}
\usepackage{thm-restate}
\declaretheorem[name=Theorem,parent=section]{theorem}
\declaretheorem[name=Lemma,parent=section]{lemma}
\declaretheorem[name=Assumption, parent=section]{assumption}
\declaretheorem[name=Condition, parent=section]{condition}
\declaretheorem[qed=$\triangleleft$,name=Example,style=definition, parent=section]{example}
\declaretheorem[name=Remark, parent=section]{remark}
\declaretheorem[name=Proposition, parent=section]{proposition}

\usepackage{crossreftools}
\newcommand{\creftitle}[1]{\crtcref{#1}}

\makeatletter
  \renewenvironment{proof}[1][Proof]%
  {%
   \par\noindent{\bfseries\upshape {#1.}\ }%
  }%
  {\qed\newline}
  \makeatother

\theoremstyle{definition}  %

\newtheorem{corollary}{Corollary}[section]

\theoremstyle{plain}
\newtheorem{definition}{Definition}[section]

\xpatchcmd{\proof}{\itshape}{\normalfont\proofnameformat}{}{}
\newcommand{\proofnameformat}{\bfseries}

\newcommand{\pfref}[1]{Proof of \cref{#1}}

\renewcommand{\eqref}[1]{\texorpdfstring{\hyperref[#1]{(\ref*{#1})}}{(\ref*{#1})}}
\crefformat{equation}{#2Eq.\,(#1)#3}
\Crefformat{equation}{#2Eq.\,(#1)#3}

\Crefformat{figure}{#2Figure~#1#3}
\Crefformat{assumption}{#2Assumption~#1#3}

\Crefname{assumption}{Assumption}{Assumptions}

\crefname{fact}{Fact}{Facts}

\Crefformat{figure}{#2Figure #1#3}
\Crefformat{assumption}{#2Assumption #1#3}

\usepackage{crossreftools}
\usepackage{xparse}

\ExplSyntaxOn
\DeclareDocumentCommand{\XDeclarePairedDelimiter}{mm}
 {
  \__egreg_delimiter_clear_keys: %
  \keys_set:nn { egreg/delimiters } { #2 }
  \use:x %
   {
    \exp_not:n {\NewDocumentCommand{#1}{sO{}m} }
     {
      \exp_not:n { \IfBooleanTF{##1} }
       {
        \exp_not:N \egreg_paired_delimiter_expand:nnnn
         { \exp_not:V \l_egreg_delimiter_left_tl }
         { \exp_not:V \l_egreg_delimiter_right_tl }
         { \exp_not:n { ##3 } }
         { \exp_not:V \l_egreg_delimiter_subscript_tl }
       }
       {
        \exp_not:N \egreg_paired_delimiter_fixed:nnnnn 
         { \exp_not:n { ##2 } }
         { \exp_not:V \l_egreg_delimiter_left_tl }
         { \exp_not:V \l_egreg_delimiter_right_tl }
         { \exp_not:n { ##3 } }
         { \exp_not:V \l_egreg_delimiter_subscript_tl }
       }
     }
   }
 }

\keys_define:nn { egreg/delimiters }
 {
  left      .tl_set:N = \l_egreg_delimiter_left_tl,
  right     .tl_set:N = \l_egreg_delimiter_right_tl,
  subscript .tl_set:N = \l_egreg_delimiter_subscript_tl,
 }

\cs_new_protected:Npn \__egreg_delimiter_clear_keys:
 {
  \keys_set:nn { egreg/delimiters } { left=.,right=.,subscript={} }
 }

\cs_new_protected:Npn \egreg_paired_delimiter_expand:nnnn #1 #2 #3 #4
 {%
  \mathopen{}
  \mathclose\c_group_begin_token
   \left#1
   #3
   \group_insert_after:N \c_group_end_token
   \right#2
   \tl_if_empty:nF {#4} { \c_math_subscript_token {#4} }
 }
\cs_new_protected:Npn \egreg_paired_delimiter_fixed:nnnnn #1 #2 #3 #4 #5
 {
  \mathopen{#1#2}#4\mathclose{#1#3}
  \tl_if_empty:nF {#5} { \c_math_subscript_token {#5} }
 }
\ExplSyntaxOff

\XDeclarePairedDelimiter{\supnorm}{
  left=\lVert,
  right=\rVert,
  subscript=\infty
  }

\DeclarePairedDelimiter{\abs}{\lvert}{\rvert} %
\DeclarePairedDelimiter{\brk}{[}{]}
\DeclarePairedDelimiter{\crl}{\{}{\}}
\DeclarePairedDelimiter{\prn}{(}{)}
\DeclarePairedDelimiter{\nrm}{\|}{\|}
\DeclarePairedDelimiter{\tri}{\langle}{\rangle}

\let\Pr\undefined

\DeclareMathOperator{\En}{\mathbb{E}}

\DeclareMathOperator{\Pr}{Pr}

\newcommand{\mb}[1]{\boldsymbol{#1}}

\def\ddefloop#1{\ifx\ddefloop#1\else\ddef{#1}\expandafter\ddefloop\fi}
\def\ddef#1{\expandafter\def\csname bb#1\endcsname{\ensuremath{\mathbb{#1}}}}
\ddefloop ABCDEFGHIJKLMNOPQRSTUVWXYZ\ddefloop
\def\ddefloop#1{\ifx\ddefloop#1\else\ddef{#1}\expandafter\ddefloop\fi}
\def\ddef#1{\expandafter\def\csname b#1\endcsname{\ensuremath{\mathbf{#1}}}}
\ddefloop ABCDEFGHIJKLMNOPQRSTUVWXYZ\ddefloop
\def\ddef#1{\expandafter\def\csname sf#1\endcsname{\ensuremath{\mathsf{#1}}}}
\ddefloop ABCDEFGHIJKLMNOPQRSTUVWXYZ\ddefloop
\def\ddef#1{\expandafter\def\csname c#1\endcsname{\ensuremath{\mathcal{#1}}}}
\ddefloop ABCDEFGHIJKLMNOPQRSTUVWXYZ\ddefloop
\def\ddef#1{\expandafter\def\csname h#1\endcsname{\ensuremath{\widehat{#1}}}}
\ddefloop ABCDEFGHIJKLMNOPQRSTUVWXYZ\ddefloop
\def\ddef#1{\expandafter\def\csname hc#1\endcsname{\ensuremath{\widehat{\mathcal{#1}}}}}
\ddefloop ABCDEFGHIJKLMNOPQRSTUVWXYZ\ddefloop
\def\ddef#1{\expandafter\def\csname t#1\endcsname{\ensuremath{\widetilde{#1}}}}
\ddefloop ABCDEFGHIJKLMNOPQRSTUVWXYZ\ddefloop
\def\ddef#1{\expandafter\def\csname tc#1\endcsname{\ensuremath{\widetilde{\mathcal{#1}}}}}
\ddefloop ABCDEFGHIJKLMNOPQRSTUVWXYZ\ddefloop
\def\ddefloop#1{\ifx\ddefloop#1\else\ddef{#1}\expandafter\ddefloop\fi}
\def\ddef#1{\expandafter\def\csname scr#1\endcsname{\ensuremath{\mathscr{#1}}}}
\ddefloop ABCDEFGHIJKLMNOPQRSTUVWXYZ\ddefloop

\newcommand{\veps}{\varepsilon}

\newcommand{\ldef}{\vcentcolon=}
\newcommand{\rdef}{=\vcentcolon}

\newcommand{\pisamp}{\wt{\pi}}
\newcommand{\samp}{\mb{\pi}_{\mathsf{samp}}}
\newcommand{\bmu}{\mb{\mu}}
\newcommand{\coefffull}{\coeff(\Pi,T,\beta;\samp)}
\newcommand{\cDopt}{\cD_{\mathsf{opt}}}

\newcommand{\Lhat}{\wh{L}}%
\newcommand{\bhat}{\wh{b}}%
\newcommand{\Bhat}{\wh{B}}%
\newcommand{\tautil}{\wt{\tau}}%
\newcommand{\stil}{\wt{s}}
\newcommand{\atil}{\wt{a}}

\newcommand{\Jb}{J_{\beta}}

\newcommand{\Ccov}{C_{\mathsf{cov}}}
\newcommand{\Cconc}{C_{\mathsf{conc}}}

\newcommand{\mainalg}{Exploratory Preference Optimization\xspace}
\newcommand{\alglong}{\mainalg}
\newcommand{\algshort}{\texttt{XPO}\xspace} %
\newcommand{\dpo}{\texttt{DPO}\xspace}

\newcommand{\taup}{\tau_{+}}
\newcommand{\taum}{\tau_{-}}

\newcommand{\pistarb}{\pistar_{\beta}}
\newcommand{\Qstarb}{Q^{\star}_{\beta}}
\newcommand{\Vstarb}{V^{\star}_{\beta}}

  \newcommand{\ap}{a_{+}}%
  \newcommand{\am}{a_{-}}%

\newcommand{\pitil}{\wt{\pi}}%

\newcommand{\Qstar}{Q^{\star}}
\newcommand{\Vstar}{V^{\star}}

  \newcommand{\afrak}{\mathfrak{a}}
  \newcommand{\bfrak}{\mathfrak{b}}

\renewcommand{\emptyset}{\varnothing}

\newcommand{\filt}{\mathscr{F}}

\newcommand{\M}[1]{^{{\scriptscriptstyle M}}}  %

\newcommand{\pistar}{\pi^{\star}}

\newcommand{\pihat}{\wh{\pi}}

\newcommand{\algcommentlight}[1]{\textcolor{blue!70!black}{\transparent{0.5}{\scriptsize{\texttt{\textbf{//\hspace{2pt}#1}}}}}}

\newcommand{\psdgt}{\succ}
\newcommand{\approxleq}{\lesssim}
\newcommand{\approxgeq}{\gtrsim}

\newcommand{\ind}[1]{^{{\scriptscriptstyle(#1)}}}

\newcommand{\bigoh}{O}
\newcommand{\bigoht}{\wt{O}}
\newcommand{\bigom}{\Omega}
\newcommand{\bigomt}{\wt{\Omega}}
\newcommand{\bigthetat}{\wt{\Theta}}

\newcommand{\indic}{\mathbb{I}}

\renewcommand{\Pr}{\bbP}

\newcommand{\poly}{\mathrm{poly}}
\newcommand{\polylog}{\mathrm{polylog}}

\newcommand{\Dhels}[2]{D^{2}_{\mathsf{H}}\prn*{#1,#2}}

\newcommand{\dmid}{\;\|\;}

\def\multiset#1#2{\ensuremath{\left(\kern-.3em\left(\genfrac{}{}{0pt}{}{#1}{#2}\right)\kern-.3em\right)}}

\renewcommand{\emptyset}{\varnothing}

\input{widebar}

\makeatletter
\let\OldStatex\Statex
\renewcommand{\Statex}[1][3]{%
  \setlength\@tempdima{\algorithmicindent}%
  \OldStatex\hskip\dimexpr#1\@tempdima\relax}
\makeatother

 \usepackage{accents}

\let\oldparagraph\paragraph

\renewcommand{\paragraph}[1]{\oldparagraph{#1.}}

\allowdisplaybreaks

\neurips{                      \title{Exploratory Preference Optimization:\\
                        {\large Harnessing Implicit $Q^\star$-Approximation for
                        Sample-Efficient RLHF} }                  
                        }
\arxiv{                      \title{\huge Exploratory Preference Optimization:\\
                        {\Large Harnessing Implicit Q*-Approximation for
                        Sample-Efficient RLHF} }                  
                        }

\author{
Tengyang Xie\thanks{Equal contribution}
\\
\normalsize
\href{mailto:tx@cs.wisc.edu}{\texttt{tx@cs.wisc.edu}}
\and
Dylan J. Foster\footnotemark[1]
\\
\normalsize
\href{mailto:dylanfoster@microsoft.com}{\texttt{dylanfoster@microsoft.com}}
\and
Akshay Krishnamurthy
\\
\normalsize
\href{mailto:akshaykr@microsoft.com}{\texttt{akshaykr@microsoft.com}}
\and
Corby Rosset
\\
\normalsize
\href{mailto:corbyrosset@microsoft.com}{\texttt{corbyrosset@microsoft.com}}
\and
Ahmed Awadallah
\\
\normalsize
\href{mailto:ahmed.awadallah@microsoft.com}{\texttt{ahmed.awadallah@microsoft.com}}
\and
Alexander Rakhlin
\\
\normalsize
\href{mailto:rakhlin@mit.edu}{\texttt{rakhlin@mit.edu}}
}

\date{May 30, 2024}

  \addtocontents{toc}{\protect\setcounter{tocdepth}{0}}

\begin{document}
\maketitle

\begin{abstract}
  Reinforcement learning from human feedback (RLHF) has emerged as a
central tool for language model alignment. We consider \emph{online exploration} in RLHF, which
exploits interactive access to human or AI feedback by deliberately
encouraging the model to produce diverse, maximally informative responses. By allowing
RLHF to confidently stray from the
pre-trained model, online exploration offers the possibility of novel, potentially super-human
capabilities, but its full potential as a paradigm for language model training has yet to
be realized, owing to computational and statistical
bottlenecks in directly adapting existing reinforcement learning techniques.\loose

We propose a new algorithm for online exploration in RLHF,
\emph{\alglong} (\algshort), which is simple and practical---a
one-line change to (online) Direct Preference Optimization \citep[\dpo;][]{rafailov2024direct}---yet enjoys
the strongest known provable guarantees
and promising empirical
    performance. \algshort augments the \dpo objective with a novel
    and principled 
    \emph{exploration bonus}, empowering the algorithm to explore
    outside the support of the initial model and human feedback
    data. In theory, we show that \algshort is provably
    sample-efficient and converges to a near-optimal language model policy
    under natural exploration conditions, irrespective of whether the
    initial model has good coverage.
    Our analysis, which builds on the observation that \dpo
    implicitly performs a form of $Q^{\star}$-approximation (or, \emph{Bellman error minimization}), combines previously disparate techniques from language modeling and theoretical reinforcement
learning in a serendipitous fashion through the
perspective of \emph{KL-regularized Markov decision processes}. Empirically, we
find that \algshort is more sample-efficient than
non-exploratory \dpo
    variants in a preliminary evaluation.\loose

\end{abstract}

\section{Introduction}
\label{sec:intro}

Reinforcement learning from human feedback (RLHF) is a central tool to align language models to human values and elicit useful
behavior
\citep{christiano2017deep,bai2022training,ouyang2022training}. Using
human-labeled preference data, RLHF achieves enhanced capabilities
using a modest amount of data compared to unsupervised
pre-training (on the order of tens of millions
versus trillions of tokens) by treating
the language model as a ``policy'' and optimizing it with reinforcement learning
techniques.\loose

Even though RLHF is typically
only applied with preference data from humans or other language
models, one might hope that it has potential to produce super-human capabilities because
recognizing novel behavior and insights is typically easier than
\emph{generating} novel behavior. Indeed, it is often much
  easier to verify correctness of a given proof or program than it is
  to produce one from scratch. By repeatedly generating
new proposals and labeling them with human feedback, a
language model could gradually push beyond the boundary of human
capabilities. Unfortunately, even with the great disparity in difficulty between generation and verification, a major barrier to achieving enhanced capabilities via RLHF is the volume of human feedback, i.e., \emph{sample complexity}, required by existing methods. Thus, a promising research direction is to develop sample-efficient methods for RLHF.

\neurips{\vspace{-2pt}}
A natural way to address the sample efficiency problem for RLHF
is to augment algorithms with \emph{online exploration}. Online exploration exploits interactive access to
human or AI feedback by deliberately encouraging the model to produce
diverse, novel responses. RLHF algorithms that exploit online feedback have received limited
investigation, and in spite of encouraging initial results, existing
approaches either do not update the language model
\citep{dwaracherla2024efficient}, or engage in purely passive
exploration \citep{guo2024direct,gao2024rebel}, with no
  mechanism to encourage novelty or diversity. Passive exploration is intuitively insufficient, as we are unlikely to generate novel and correct proofs by chance; we make this precise in \Cref{prop:dpo_failure}.
  Thus, the full potential of online
exploration as a new paradigm for language model training has yet to
be realized. 

\neurips{\vspace{-2pt}}
The central challenge in equipping language models with deliberate exploration is to efficiently navigate the vast, combinatorially large
space of token sequences to find responses for which feedback will be
maximally informative. %
The contemporary theory of reinforcement learning offers---at a
conceptual level---solutions to this problem, providing algorithm
design principles for exploration that can optimally take advantage of
problem structure and achieve sample efficiency to the best extent one
can hope for
\citep{jiang2017contextual,agarwal2019reinforcement,foster2023foundations}. However,
the most powerful approaches in this space are computationally
intractable in the general \arxiv{reinforcement
    learning}\neurips{RL} setting \citep{jiang2017contextual,jin2021bellman,foster2021statistical}, and prior attempts to adapt them to
  RLHF either make unrealistic modeling assumptions (i.e., do not allow for
general function approximation)
\citep{xu2020preference,novoseller2020dueling,pacchiano2021dueling,wu2023making,zhan2023query,du2024exploration,das2024provably}, or are
  computationally inefficient and not feasible to faithfully implement
  \citep{chen2022human,wang2023rlhf,ye2024theoretical}.
Can we, perhaps by specializing to language modeling, develop practical, provable, and empirically efficient online exploration methods for RLHF?\loose

\neurips{\vspace{-3pt}}
\subsection{Contributions}
We propose a new algorithm for online exploration in RLHF,
\emph{\alglong (\algshort)}, which is simple and practical---a
one-line change to (online) Direct Preference Optimization (\dpo; \citet{rafailov2024direct,guo2024direct})---yet enjoys
the strongest known provable
  guarantees and promising empirical
    performance. \algshort augments the \dpo objective with a novel
    and principled 
    \emph{exploration bonus}, empowering the algorithm to explore
    outside the support of the initial model. We show that \algshort is provably
    sample-efficient, and converges to a near-optimal language model policy
    under natural exploration conditions
    \citep{jin2021bellman,xie2023role,zhong2022gec}. Critically, and in contrast to
    prior work, our theory holds irrespective of whether the initial
    model is sufficiently exploratory on its
    own. To summarize:\loose
    \begin{center}
      \emph{\algshort offers the first practical and provably
        sample-efficient online exploration algorithm for RLHF with
        general function approximation.}
    \end{center}
    
\neurips{\vspace{-1pt}}

\arxiv{\paragraph{Technical highlights}}
\neurips{\textbf{Technical highlights.}~~}
Our design and analysis of \algshort uses previously disparate
techniques from language modeling and theoretical reinforcement
learning, combining them in a serendipitous fashion through the
perspective of \emph{KL-regularized Markov decision processes} \citep{neu2017unified}.
\begin{enumerate}
\item First, generalizing \citet{rafailov2024r}, we observe that \dpo
        can be viewed as implicitly performing \emph{Bellman error minimization}~\citep{xie2020q} to approximate the
        optimal value function $Q^{\star}$ in a \emph{KL-regularized
          MDP}. We use this to provide a novel KL-regularized regret decomposition.
      \item Then, we show that \emph{global optimism}
        \citep{jiang2017contextual,jin2021bellman,xie2023role}, a
        powerful RL exploration technique that has classically been
        viewed as
        computationally intractable~\citep{dann2018oracle,kane2022computational,golowich2024exploration},
        can be implemented in any
        KL-regularized MDP with deterministic transitions (generalizing language modeling) by adding a surprisingly simple exploration bonus to the \dpo objective. This yields the \algshort objective.\loose
    \end{enumerate}
    \neurips{\vspace{-3pt}}
        We expect our analysis techniques and perspective to be useful
        more broadly. In particular, the guarantees for \algshort hold not
just for language models, but for any \arxiv{reinforcement learning}\neurips{RL} problem
with a stochastic starting state and (potentially unknown)
deterministic transition dynamics (``Deterministic Contextual MDP'').\loose

\arxiv{\paragraph{Empirical results}}
\neurips{\textbf{Empirical results.}~~}
In \cref{sec:experiments}, we perform a proof-of-concept experiment to
validate our theory,
and find that \algshort can match the performance of \dpo
variants \citep{xu2023some,snorkelai2024,dong2024rlhf} based on passive or
heuristic exploration using
significantly less preference data.
These initial
findings suggest that augmenting language models with online
exploration may indeed lead to benefits over passive
  exploration.\loose

\paragraph{Concurrent work}  
Two concurrent and independent works posted to arXiv just before this
preprint, \citet{cen2024value,zhang2024self}, propose algorithms that
equip \dpo with exploration bonuses similar to \algshort. On the theoretical side, both works are
restricted to the contextual bandit formulation of RLHF, and do not
consider the general reinforcement learning framework in this work or
make the connection to $Q^{\star}$-approximation and KL-regularized MDPs. Compared to our
results, which give provable sample complexity guarantees with general
function approximation, \citet{zhang2024self}
do not provide sample complexity guarantees, while
\citet{cen2024value} provide guarantees only for linear contextual
bandits. In addition, and importantly, the sample complexity guarantees in
\citet{cen2024value} have exponential dependence on the KL
regularization parameter, which our results avoid. Empirically, both works find benefits from exploration.\loose

\arxiv{\subsection{Paper Organization}
\cref{sec:background} presents background on RLHF, online feedback, and the necessity of exploration. \cref{sec:main} presents our
algorithm and main theoretical guarantees\neurips{.}\arxiv{, including motivation behind the
algorithm design and a proof sketch.} \cref{sec:experiments}
presents experimental results, and we conclude with discussion in
\cref{sec:discussion}. Proofs and additional results are deferred to
the appendix.\loose

\paragraph{Notation}
  For an integer $n\in\bbN$, we let $[n]$ denote the set
  $\{1,\dots,n\}$. For a set $\cX$, we let $\Delta(\cX)$ denote the
  set of all probability distributions over $\cX$. We adopt
    standard big-oh notation, and write $f=\bigoht(g)$ to denote that
    $f = \bigoh(g\cdot{}\max\crl*{1,\mathrm{polylog}(g)})$ and
    $a\approxleq{}b$ as shorthand for $a=\bigoh(b)$. \loose
  }

\neurips{\vspace{-3pt}}
\section{Background}
\label{sec:background}
This section contains necessary background to present our main
results. We begin by recalling the standard formulation of reinforcement learning
from human feedback from offline data (\cref{sec:rlhf}), then introduce the \emph{online
  feedback} model and highlight the need for systematic
exploration (\cref{sec:online_background}).\loose

\neurips{
\paragraph{Notation}
  For an integer $n\in\bbN$, we let $[n]$ denote the set
  $\{1,\dots,n\}$. For a set $\cX$, we let $\Delta(\cX)$ denote the
  set of all probability distributions over $\cX$. \arxiv{We adopt
    non-asymptotic big-oh notation: For functions
    $f,g:\cX\to\bbR_{+}$, we write $f=\bigoh(g)$ (resp. $f=\bigom(g)$)
    if there exists a constant $C>0$ such that $f(x)\leq{}Cg(x)$
    (resp. $f(x)\geq{}Cg(x)$) for all $x\in\cX$. We write
    $f=\bigoht(g)$ if $f=\bigoh(g\cdot\mathrm{polylog}(T))$,
    $f=\bigomt(g)$ if $f=\bigom(g/\polylog(T))$, and $f=\bigthetat(g)$
    if $f=\bigoht(g)$ and
    $f=\bigomt(g)$. %
    We write $f\propto g$ if $f=\bigthetat(g)$.} \neurips{We adopt
    standard big-oh notation, and write $f=\bigoht(g)$ to denote that
    $f = \bigoh(g\cdot{}\max\crl*{1,\mathrm{polylog}(g)})$ and
    $a\approxleq{}b$ as shorthand for $a=\bigoh(b)$.} \loose
  }

\subsection{Reinforcement Learning from Human Feedback}
\label{sec:rlhf}

We study RLHF in a general reinforcement learning formulation which
subsumes the \emph{token-level MDP} formulation considered in prior
work \citep{rafailov2024r}, but is somewhat broader.

\paragraph{Markov decision processes}
We consider an episodic finite-horizon Markov decision process framework. Formally, a horizon-$H$
MDP $M=(H,\cS, \cA, P, r,\rho)$ consists of a
(potentially very large) state space $\cS$, action space $\cA$, probability transition function
$P:\cS\times\cA\to\Delta(\cS)$, reward function
$r:\cS\times{}\cA\to\bbR$, and initial state distribution
$\rho\in\Delta(\cS)$. 
We assume
without loss of generality that the state space is \emph{layered} such that $\cS=\cS_1\cup\cS_2\cup\cdots\cup{}\cS_H$, where
$\cS_h$ is the set of states reachable at step $h$, and
$\cS_h\cup{}\cS_{h'}=\emptyset$ for $h\neq{}h'$. A (randomized) policy
is a mapping 
$\pi:\cS\to\Delta(\cA)$, and induces a
distribution over trajectories $\tau=(s_1,a_1),\ldots,(s_H,a_H)$ and
rewards $r_1,\ldots,r_H$ via the following process. The initial state is drawn via
$s_1\sim{}\rho$, then for $h=1,\ldots,H$: $a_h\sim\pi(s_h)$,
$r_h=r(s_h,a_h)$, and $s_{h+1}\sim{}P(s_h,a_h)$. We let $\En_{\pi}\brk*{\cdot}$ and $\Pr_{\pi}\brk{\cdot}$
  denote expectation and probability under this process,
  respectively. We assume that $\sum_{h=1}^{H}r_h\in\brk{0,\Rmax}$
  almost surely for a parameter $\Rmax>0$. For a trajectory $\tau$ and policy $\pi$ we define $r(\tau) = \sum_{h=1}^H r(s_h,a_h)$ and $\pi(\tau) = \prod_{h=1}^H \pi(a_h \mid s_h)$.

  In the context of language modeling, the main object of interest is
  the \emph{token-level MDP} \citep{rafailov2024r}. Here, $s_1\sim\rho$ represents a prompt,
  each action $a_h$ represents a token (with $\cA$ representing the vocabulary), and the state
  $s_h=(s_1,a_1,\ldots,a_{h-1})$ is the prompt and sequence of
  tokens so far. The language model is represented by a policy $\pi$,
  which maps the current context $s_h=(s_1,a_1,\ldots,a_{h-1})$ to a
  distribution over the next token $a_h$. The trajectory
  $\tau=(s_1,a_1),\ldots,(s_H,a_H)$ produced by this process
  can be interpreted as  the language model's response to the prompt
  $s_1$; we will occasionally use the terms ``trajectory'' and
  ``response'' synonymously in this context.

  Our main results apply to any \emph{Deterministic Contextual MDP} (DCMDP)
  for which the initial state is stochastic, but the subsequent
  transition dynamics are deterministic and potentially unknown. This
  formulation encompasses but strictly generalizes the token-level MDP.

  \paragraph{RLHF with offline data}
In the classical RLHF formulation \citep{christiano2017deep,bai2022training,ouyang2022training}, we assume access to a dataset
$\cD_\pref=\crl*{(\taup,\taum)}$ of labeled preference data. Each pair
of trajectories (responses) $(\taup,\taum)$ represents a positive and
negative example; both trajectories begin from the
same initial state (prompt) $s_1$, and are generated by first sampling
a pair $(\tau,\tautil)$ via $\tau\sim{}\piref\mid{}s_1 $ and
$\tautil\sim{}\piref\mid{}s_1$ in the
underlying DCMDP $M$ (e.g., token-level MDP), and then ordering them
as $(\taup,\taum)$ based on a binary preference
$y\sim{}\bbP(\tau\psdgt{}\tautil\mid{}s_1)$. Here,
$\piref$ is a \emph{reference policy} (language model), which is
typically obtained via supervised fine-tuning, and the
\emph{preference} $y\sim{}\bbP(\tau\psdgt{}\tautil\mid{}s_1)$ is obtained from a human or AI
annotator. Following a standard assumption
\citep{christiano2017deep,rafailov2024direct,rafailov2024r}, we assume
that preferences follow the \emph{Bradley-Terry} model
\citep{bradley1952rank}: For trajectories $\tau$ and $\wt{\tau}$ both beginning with $s_1$,\loose 
\begin{align}
\label{eq:bt}
\bbP(\tau\psdgt\tautil\mid{}s_1) = \frac{\exp\prn*{r(\tau)}}{\exp\prn*{r(\tau)} + \exp\prn*{r(\wt{\tau})}}.
\end{align}

Based on the preference dataset $\cD_\pref$, the goal is to learn a
policy $\pihat$ with high reward. Following prior theoretical works on
RLHF, we consider a \emph{KL-regularized} reward
objective \citep{xiong2023gibbs,ye2024theoretical}, defined for a regularization parameter $\beta>0$, via
\neurips{\begin{small}}
\begin{align}
    \label{eq:kl_reward}
    J_\beta(\pi) \coloneqq &~ J(\pi) -
                             \beta\cdot\sum_{h=1}^{H}\En_{\pi}\brk*{D_{\mathrm{KL}}{(\pi(\cdot\mid{}s_h)\dmid\piref(\cdot\mid{}s_h))}}
                             = \E_{\pi} \left[r(\tau) - \beta \log \frac{\pi(\tau)}{\piref(\tau)}\right].
\end{align}
\neurips{\end{small}}%
We aim to compute a policy $\pihat$ such
that \neurips{$\max_{\pi}J_\beta(\pi) - J_\beta(\pihat) \leq \veps$}
\arxiv{\[
\max_{\pi}J_\beta(\pi) - J_\beta(\pihat) \leq \veps
\]}
for some small $\veps>0$. Such a guarantee means that 
$\pihat$ near-optimally maximizes reward, yet stays relatively close
to $\piref$ (as a function of $\beta$). The choice of $\beta>0$, which is important for safety and reliability, is typically viewed as a
domain specific hyperparameter \citep{tang2024understanding}. Our main
focus in this paper is the \emph{small-$\beta$} regime, which allows
$\pihat$ to meaningfully deviate from $\piref$ and generate
potentially novel responses. Notably, by taking $\beta$
sufficiently small, it is possible to translate suboptimality bounds
for the regularized reward into bounds for the unregularized reward (e.g., \citealp{zhu2023principled,zhan2023provable}).\loose

We refer to this setting as \emph{offline RLHF} because the algorithm
 relies only on the offline dataset $\cD_\pref$ for training, and
does not perform any active data collection.

\paragraph{Direct preference optimization (\dpo)}
Initial approaches to offline RLHF \citep{christiano2017deep,ouyang2022training} proceed by first
estimating a reward function $\wh{r}$ from $\cD_\pref$ using the Bradley-Terry model, then optimizing
an estimated version of the KL-regularized objective in
\cref{eq:kl_reward} using policy optimization methods like
PPO, i.e., \neurips{$\pihat\approx\argmax_{\pi\in\Pi}\E_{\pi} \brk[\big]{ r(\tau) - \beta \log\frac{\pi(\tau)}{\piref(\tau)}} $.}
\arxiv{\begin{align}
  \label{eq:ppo}
\pihat\approx\argmax_{\pi\in\Pi}\E_{\pi} \left[ \sum_{h=1}^H
                                        \left( \rhat(s_h,a_h) - \beta
                                        \log\frac{\pi(a_h \mid
                                        s_h)}{\piref(a_h \mid s_h)}
                                        \right) %
                                        \right].
\end{align}}
The starting point for our work is an alternative approach
introduced by \citet{rafailov2024direct}, Direct Preference
Optimization (\dpo). \dpo is motivated by a closed-form
solution for the policy that optimizes the KL-regularized objective in \cref{eq:kl_reward}, and condenses the two-step process above into a
single policy optimization objective, removing the need for reward
function estimation. Concretely, \dpo solves\footnote{We adopt the
  convention that the value of the \dpo objective is $+\infty$ if
  $\pi$ does not satisfy $\pi\ll{}\piref$.} \loose
\begin{align}
  \label{eq:dpo}
  \pihat=\argmin_{\pi\in\Pi}
\sum_{(\tau_+,\tau_-) \in \Dcal_\pref} - \log\left[\sigma\left( \beta\log\frac{\pi(\tau_+)}{\piref(\tau_+)} - \beta\log\frac{\pi(\tau_-)}{\piref(\tau_-)} \right) \right] 
\end{align}
for a user-specified policy class $\Pi$, where
$\sigma(x)\ldef{}\frac{\exp(x)}{1+\exp(x)}$ is the sigmoid function.

\subsection{Online Feedback and Exploration in RLHF}
\label{sec:online_background}
\dpo and other offline RLHF methods have achieved great
success in language model alignment, but
are fundamentally limited to behaviors that are well-supported by the initial
model $\piref$ and \arxiv{preference }data $\cD_\pref$. RLHF with \emph{online
  feedback} offers a promising approach to move beyond this
limitation by collecting feedback from responses sampled from the
model \emph{during training} \citep{guo2024direct}.\loose

Formally, the protocol proceeds in $T$
rounds. At each round $t$, we receive an initial state
$s_1\ind{t}$ and sample two responses $\tau\sim{}\pi\ind{t}\mid{}s_1 $
and $\wt{\tau}\sim{}\pi\ind{t}\mid{}s_1$ from the current policy
$\pi\ind{t}$. The prompts are then labeled as 
$(\taup\ind{t},\taum\ind{t})$ and added to the preference dataset via
$\cD_\pref\ind{t+1}\gets\cD_\pref\ind{t}\cup\crl{(\taup\ind{t},\taum\ind{t})}$,
which is then used to compute an updated policy $\pi\ind{t+1}$. In
practice, the prompts are typically labeled via human feedback or AI
feedback (e.g., a larger, more powerful language model \citep{guo2024direct,rosset2024direct}); we
assume the preferences $\bbP(\tau\ind{t}\psdgt{}\wt{\tau}\ind{t}\mid{}s_1\ind{t})
$ follow the Bradley-Terry model in \cref{eq:bt}.\loose

\subsection{The Necessity of Deliberate Exploration}

Existing approaches to online RLHF adapt offline techniques by
applying them iteratively. As an example, \emph{Online \dpo}
\citep{guo2024direct} proceeds as follows:\footnote{The closely
  related \emph{Iterative \dpo} approach \citep{xu2023some,snorkelai2024} proceeds in the same
  fashion, but samples a large batch of preference pairs from each
 policy $\pi\ind{t}$\arxiv{ instead of a single pair}, and performs fewer updates.\loose}\loose
\begin{enumerate}
\item Compute $\pi\ind{t}$ by solving the \dpo{} objective in
  \cref{eq:dpo} with the current preference dataset $\cD_\pref\ind{t}$.
\item Sample
  $\tau\ind{t},\wt{\tau}\ind{t}\sim{}\pi\ind{t}\mid{}s_1\ind{t}$, then
  label as $(\taup\ind{t},\taum\ind{t})$ and update  $\cD_\pref\ind{t+1}\gets\cD_\pref\ind{t}\cup\crl{(\taup\ind{t},\taum\ind{t})}$.\loose
\end{enumerate}

We refer to such an approach as \emph{passive exploration}, as the
responses are sampled directly from the policy $\pi\ind{t}$ without an
explicit mechanism to encourage diversity. The following proposition
shows that passive exploration is insufficient to discover novel behavior:
Unless the initial policy $\piref$ has good coverage, Online DPO can
fail to learn a near-optimal policy.
  \begin{proposition}[Necessity of deliberate exploration]
    \label{prop:dpo_failure}
    Fix $\beta\in(0,\tfrac{1}{8}\log(2))$, and consider the
    \emph{bandit} setting ($H=1$, $\cS=\emptyset$, and $\abs*{\cA}=2$). There exists\arxiv{ a
    reference policy} $\piref$ such that for all
    $T\leq{}\tfrac{1}{2}\exp(\frac{1}{8\beta})$, with constant
    probability, all of the policies
    $\pi\ind{1},\ldots,\pi\ind{T+1}$ produced by Online \dpo satisfy\loose
    \begin{align}
      \max_{\pi}J_\beta(\pi) - J_\beta(\pi\ind{t}) \geq{} \frac{1}{8}\quad\forall{}t\in\brk{T+1}.
    \end{align}
\end{proposition}
That is, the sample complexity required by Online \dpo is
\emph{exponential} in $\frac{1}{\beta}$, which is unacceptable in the
small-$\beta$ regime; inspecting the proof, it is straightforward to see that the same
conclusion holds for Iterative \dpo and purely offline \dpo.
The idea behind \cref{prop:dpo_failure} is simple: If $\piref$ places
small probability mass on the optimal action, Online \dpo may fail to
ever explore this action until the number of iterations is
exponentially large. This reflects the intuition that in the
  small-$\beta$ regime, more deliberate
exploration is required to discover behaviors or capabilities not
already covered by $\piref$. 

\begin{remark}
  Various empirical works have suggested that offline \dpo can
  under-perform relative to vanilla RLHF with PPO due to a lack of
  on-policy sampling \citep{xiong2023gibbs,guo2024direct,dong2024rlhf,tang2024understanding}. \cref{prop:dpo_failure} highlights a conceptually
  distinct phenomenon, where both of the aforementioned algorithms (as
  well as online variants of \dpo) fail due to poor coverage from
  $\piref$, in spite of on-policy sampling.\loose
\end{remark}

\neurips{\section{\mbox{Online Exploration for Language Models:
      \alglong}}}
\arxiv{\section{\alglong}}
\label{sec:main}

We now present our main algorithm \algshort, which addresses the
limitations of existing alignment methods by augmenting \dpo with
active exploration. We first describe the algorithm and motivation
(\cref{sec:algorithm}), then present theoretical guarantees
(\cref{sec:theoretical}), and sketch the analysis
(\cref{sec:proof_sketch}). \loose

  \subsection{The \algshort Algorithm}
  \label{sec:algorithm}

\begin{algorithm}[ht]
\caption{\mainalg (\algshort)}
\label{alg:opt_dpo}
\begin{adjustbox}{max width=\textwidth}
\begin{minipage}{\linewidth}
\begin{algorithmic}[1]
  \Statex[0] \mbox{{\bfseries input:}
    Number of iterations $T$, KL-regularization coefficient $\beta>0$, optimism coefficient $\alpha>0$.}%
\State Initialize $\pi^\iter{1} \leftarrow \piref$, $\cD_\pref\ind{0}\gets\emptyset$.
\For{iteration $t = 1,2,\dotsc,T$}
    \State \textbf{Generate response pair
      $(\tau^\iter{t},\taut^\iter{t})$ via:} $s\ind{t}_1 \sim \rho$, $\tau^\iter{t} \sim \pi^\iter{t} \mid s\ind{t}_1$, and $\taut^\iter{t} \sim \piref \mid s\ind{t}_1$.
    \State \textbf{Label with preference:} Label $(\tau^\iter{t},\taut^\iter{t})$ as
      $(\tau^\iter{t}_+,\tau^\iter{t}_-)$ with preference $y\ind{t}\sim{}\Pr(\tau^\iter{t} \succ \taut^\iter{t})$.
    \State \textbf{Update preference data:} $\Dcal_\pref^\iter{t} \leftarrow \Dcal_\pref^\iter{t-1} \bigcup \{(\tau^\iter{t}_+,\tau^\iter{t}_-)\}$.
    \State \label{line:opt_dpo}%
    \textbf{Direct preference optimization with global optimism:}
    Calculate $\pi^\iter{t+1}$ via
    \begin{small}
      \begin{align*}
        \pi^\iter{t+1} \leftarrow \argmin_{\pi \in \Pi}
        \crl*{%
        \alpha\sum_{i=1}^{t}\log\pi(\tautil\ind{i}) -\sum_{(\tau_+,\tau_-) \in \Dcal_\pref^\iter{t}}\log\left[\sigma\left( \beta\log\frac{\pi(\tau_+)}{\piref(\tau_+)} - \beta\log\frac{\pi(\tau_-)}{\piref(\tau_-)} \right) \right]}.
      \end{align*}
    \end{small}
    \EndFor
    \State \textbf{return:}
    $\pihat=\argmax_{\pi\in\crl{\pi\ind{1},\ldots,\pi\ind{T+1}}}J_\beta(\pi\ind{t})$.\hfill\algcommentlight{Can
      compute using validation data.}
\end{algorithmic}
\end{minipage}
\end{adjustbox}
\end{algorithm}

\algshort (\alglong) is displayed in
\cref{alg:opt_dpo}. The algorithm takes as input a user-specified
policy class $\Pi$ and proceeds in almost the same fashion as Online
\dpo. For each step $t\in\brk{T}$, given the current policy
$\pi\ind{t}$ and an initial state $s_1\ind{t}$, the algorithm
begins by sampling a pair of trajectories
$\tau\ind{t}\sim\pi\ind{t}\mid{}s_1\ind{t}$ and $\tautil\ind{t}\sim\piref\mid{}s_1\ind{t}$, which
are labeled as $(\taup\ind{t},\taum\ind{t})$ based on the
preference feedback and used to update the preference dataset via
$\cD_\pref\ind{t+1}\gets\cD_\pref\ind{t}\cup\crl{(\taup\ind{t},\taum\ind{t})}$. The
most important step is \cref{line:opt_dpo}, which updates the policy
to $\pi\ind{t+1}$ via the following \emph{optimistic} variant of the
\dpo objective:\loose
\begin{small}
  \begin{align}
    \label{eq:opt_dpo}
        \pi^\iter{t+1} \leftarrow \argmin_{\pi \in \Pi}
    \crl*{%
    \alpha\sum_{i=1}^{t}\log\pi(\tautil\ind{i}) -\sum_{(\tau_+,\tau_-) \in \Dcal_\pref^\iter{t}}\log\left[\sigma\left( \beta\log\frac{\pi(\tau_+)}{\piref(\tau_+)} - \beta\log\frac{\pi(\tau_-)}{\piref(\tau_-)} \right) \right]}.
  \end{align}\end{small}%
Here, $\alpha\geq{}0$ is an \emph{optimism parameter}; for $\alpha=0$,
the algorithm nearly equivalent to Online \dpo, except that we
  sample $\tau\ind{t}\sim\pi\ind{t}\mid{}s_1\ind{t}$ and $\tautil\ind{t}\sim\piref\mid{}s_1\ind{t}$ instead of
  sampling $(\tau\ind{t},\tautil\ind{t})\sim\pi\ind{t}\mid{}s_1\ind{t}$ at each iteration.
As we will see now, for
$\alpha>0$, the term
\begin{equation}
  \label{eq:bonus}
  \alpha\sum_{i=1}^{t}\log\pi(\tautil\ind{i})
\end{equation}
in \cref{eq:opt_dpo} encourages the policy to behave
\emph{optimistically}, and produce diverse
responses $\tau$.

\paragraph{Motivation}
\emph{Optimism in the face of uncertainty} is a widely used technique in reinforcement learning
theory
\citep{agarwal2019reinforcement,lattimore2020bandit,foster2023foundations}. In
its most standard form, the optimism principle is usually stated as follows:
\emph{One should explore by choosing their actions according to the
  most optimistic view of the world, given all of the data that has
  already been observed.} The idea is that if we choose a
decision according to this principle, one of two good things can
happen: (i) the optimistic view is correct, and we receive large
reward; or (ii) the optimistic view is incorrect, but we receive useful
information that will help to better estimate the state of the world
in subsequent iterations.

Optimism is typically implemented by
directly estimating rewards, and it is not obvious at first glance why
\cref{eq:bonus} can even be interpreted as a form of optimism. To
understand, this consider a log-linear policy $\pi_f(a_h\mid{}s_h) =
\piref(a_h\mid{}s_h)\exp\prn*{\frac{f(s_h,a_h)-V_f(s_h)}{\beta}}$, where $V_f(s_h)
\coloneqq ~ \beta \log \sum_{a_h \in \Acal} \piref(a_h \mid s_h)
e^{\nicefrac{f(s_h,a_h)}{\beta}}$. Define $\brk*{\Tcal_\beta
  f}(s_h,a_h) \coloneqq ~ r(s_h,a_h) + \E\left[ V_f(s_{h+1}) \mid s_h,
a_h\right]$ as the KL-regularized Bellman operator \citep{ziebart2008maximum,ziebart2010modeling}. We observe, generalizing
\citet{rafailov2024r}, that for any DCMDP, for all trajectories
$\tau=(s_1,a_1),\ldots,(s_H,a_H)$, 
\begin{equation}
  \label{eq:reward_v_general}
  \beta\log\frac{\pi_f(\tau)}{\piref(\tau)} = r(\tau)
      -V_f(s_1) + \sum_{h=1}^H 
\left(f(s_h,a_h) - \brk*{\Tcal_\beta f}(s_h,a_h)\right).
\end{equation}
That is, the policy
can be viewed as maintaining an internal model for the trajectory reward, up to (i) a constant offset $V_f(s_1)$ that depends only on
$s_1$; and (ii) the sum of \emph{Bellman errors} $\left(f(s_h,a_h) -
  \brk*{\Tcal_\beta f}(s_h,a_h)\right)$. The optimal
KL-regularized policy $\pistarb=\argmax_{\pi}J_\beta(\pi)$ satisfies
$\pistarb=\pi_{\Qstarb}$, where $\Qstarb$/$\Vstarb$ denote KL-regularized value functions (see \cref{sec:regularized} for formal
definitions and details), and has zero Bellman error ($\Qstarb=\brk{\cT_\beta\Qstarb}$), so that\loose
  \begin{equation}
    \label{eq:reward_v}
    \beta\log\frac{\pistarb(\tau)}{\piref(\tau)} = r(\tau)
    -\Vstar_\beta(s_1)\quad\forall\tau.
  \end{equation}
In other words, $\pistarb$ implements an accurate internal reward model.
From this viewpoint:
\begin{enumerate}
\item The standard \dpo term in \cref{eq:opt_dpo} encourages the
  policy $\pi$ to build an accurate internal model for rewards
  under the Bradley-Terry model; this can be viewed as a form of
  \emph{implicit $Q^{\star}$-approximation}, since we are
    implicitly minimizing the Bellman errors in \cref{eq:reward_v_general}.
\item In light of \Cref{eq:reward_v} it is natural to approximate
  $V_{\beta}^{\pi}(s_1)$, the regularized value function for $\pi$, by
  $r(\tau) - \beta \log \frac{\pi(\tau)}{\piref(\tau)}$. Using this
  approximation, the first term in \Cref{eq:opt_dpo} biases the policy
  toward a large value function such that $\Vstarb \lesssim
  V^{\pi}_{\beta}$, implementing \emph{implicit (global) optimism} in the face of
  uncertainty (up to an inconsequential difference in on-policy
  rewards). The fact that this suffices to drive exploration is quite subtle, and leverages
  non-trivial properties of the KL-regularized MDP, including the fact
  that \cref{eq:reward_v_general} holds on a \emph{per-trajectory} basis.
\end{enumerate}

\paragraph{On the sampling policy}
As remarked above, another difference between \algshort and
online/iterative \dpo is that instead of sampling the preference
pairs via $(\tau\ind{t},\tautil\ind{t})\sim\pi\ind{t}$, we sample
$\tau\ind{t}\sim\pi\ind{t}\mid{}s_1\ind{t}$ and
$\tautil\ind{t}\sim\piref\mid{}s_1\ind{t}$. This small change is important: it is possible to show that in general, sampling
$(\tau\ind{t},\tautil\ind{t})\sim\pi\ind{t}$ can lead to degenerate
behavior in which the algorithm fails to adequately explore in the
small-$\beta$ regime, even when $\piref$ itself has good coverage.

While we use $\tautil\ind{t}\sim\piref\mid{}s_1\ind{t}$ in
\cref{alg:opt_dpo}, \algshort is significantly more general, and leads
to provable guarantees for any fixed sampling policy
$\tautil\ind{t}\sim\pitil\mid{}s_1\ind{t}$, as well as certain
data-dependent sampling schemes (e.g., sampling
$\tautil\ind{t}\sim\unif(\pi\ind{1},\ldots,\pi\ind{t})\mid{}s_1\ind{t}$);
different choices may have different tradeoffs and benefits in
practice. A general version of \algshort which leaves the sampling
distribution for $\tautil\ind{t}$ as a free parameter is given in
\cref{sec:general_alg} (\cref{alg:general}).

\arxiv{\paragraph{Practicality}}
\neurips{\textbf{Practicality.}~~}
\algshort is highly practical, and can easily be incorporated into
existing language modeling and RLHF pipelines as a drop-in replacement
for Online \dpo (a one-line change to existing code). The theoretical
guarantees for the algorithm continue to hold under standard modifications such as (i) incorporating additional preference data
from $\piref$ or another reference policy; and (ii) performing a
smaller number of iterations, but collecting a larger batch of
preference data from $\pi\ind{t}$ (as in Iterative \dpo).\loose

  \subsection{Theoretical Guarantees}
  \label{sec:theoretical}

To provide sample complexity guarantees for \algshort, we make some
standard statistical assumptions. The first\arxiv{ assumption}
asserts that the policy class $\Pi$ is powerful enough to represent
the optimal KL-regularized policy.\loose
\begin{assumption}[Policy realizability]
  \label{ass:realizability}
  The policy class $\Pi$ satisfies $\pistarb\in\Pi$.
\end{assumption}
Policy realizability is a minimal assumption for sample-efficient
reinforcement learning
\citep{agarwal2019reinforcement,lattimore2020bandit,foster2023foundations};
through \cref{eq:reward_v}, it is equivalent to a form of reward/value realizability.
For language modeling, $\Pi$ will typically correspond to a class of language
models with fixed architecture but variable weights. Next, we make a regularity assumption on the
policies in $\Pi$ \citep{rosset2024direct}.
\begin{assumption}[Bounded density ratios]
  \label{ass:vmax}
For all $\pi\in\Pi$ and trajectories
  \neurips{$\tau$}\arxiv{$\tau=(s_1,a_1),\ldots,(s_H,a_H)$},
  \begin{align}
    \abs*{\log\prn*{\frac{\pi(\tau)}{\piref(\tau)}}} \leq\frac{\Vmax}{\beta}.
         \end{align}
\end{assumption}
\arxiv{Note that }$\Vmax$ is measurable and controllable in practice; our
guarantees scale polynomially with this parameter. For log-linear policies where
  $\pi(a\mid{}s)\propto\exp\prn*{\nicefrac{f(s,a)}{\beta}}$, we 
  expect $\Vmax\approxleq{}\Rmax$.

To quantify the rate at which the algorithm converges to an
optimal policy, we require an \emph{exploration condition}, which
limits the amount of times the algorithm can be surprised by
substantially new state distributions; such assumptions are necessary
for reinforcement learning with general function approximation
\citep{jiang2017contextual,jin2021bellman,xie2023role}. Our main
result is stated in terms of a condition known as \emph{coverability} \citep{xie2023role},
but more general guarantees are given in
\cref{sec:proofs_main}. Define $d^{\pi}(\tau)\ldef{}\bbP_{\pi}((s_1,a_1),\ldots,(s_H,a_H)=\tau)$.%
\begin{definition}[Coverability]
\label{def:coverability}
  The trajectory-level coverability coefficient is given by
  \begin{align}
    \Ccov(\Pi)
    \ldef{} \inf_{\mu\in\Delta((\cS\times\cA)^{H})}\sup_{\tau\in(\cS\times\cA)^{H}}\sup_{\pi\in\Pi}\frac{d^{\pi}(\tau)}{\mu(\tau)}.
  \end{align}
  \neurips{\vspace{-10pt}}
\end{definition}%
\cref{ass:vmax} implies a trivial bound of $\Ccov(\Pi)\approxleq{}\exp\prn[\big]{\frac{\Vmax}{\beta}}$. 
Indeed, $\Ccov(\Pi)$ measures coverage with respect to the best possible distribution $\mu$, while the bound implied by \cref{ass:vmax} takes $\mu = \piref$, so we expect $\Ccov(\Pi) \ll \exp(\Vmax/\beta)$ when $\piref$ does not provide adequate coverage on its own (e.g., the example in \cref{prop:dpo_failure}). This is precisely the setting where we expect deliberate exploration to be helpful. We also note that there is a trivial bound $\Ccov(\Pi) \leq |\cA|^H$,  but because coverability depends on the structure of the (restricted) class $\Pi$, the value can be significantly smaller in general (e.g., if policies $\pi \in \Pi$ are highly correlated or stochastic). %

The
main sample complexity guarantee for \algshort is as follows.\loose
\begin{theorem}[Sample complexity bound for \algshort]
  \label{thm:main}
  Suppose that \cref{ass:realizability,ass:vmax} hold.
  For any $\beta>0$ and $T\in\bbN$, if we set
    $\alpha
  =c\cdot{}\frac{\beta}{(\Vmax+\Rmax)e^{2\Rmax}}\cdot\sqrt{\frac{
      \log(\abs{\Pi}T\delta^{-1})}{T\cdot{}\Ccov(\Pi)}}$
  for an absolute constant $c>0$,
then \cref{alg:opt_dpo} ensures that
  with probability at least $1-\delta$,\footnote{Exponential dependence on the reward range $\Rmax$ is
    an intrinsic feature of the Bradley-Terry model, and can be found
    in all prior sample complexity guarantees for this framework,
    offline and online
    \citep{das2024provably,rosset2024direct}.}
  \begin{align*}
    \Jb(\pistarb) - \Jb(\pihat)
    \approxleq
    (\Vmax+\Rmax)e^{2\Rmax}\cdot\sqrt{\frac{\Ccov(\Pi)\log(\abs{\Pi}T\delta^{-1})\log^2(T)}{T}}.
  \end{align*}
\end{theorem}
Let us discuss some key features of this result.

\arxiv{\paragraph{Statistical efficiency}}
\neurips{\textbf{Statistical efficiency.}~~}
\cref{thm:main} shows that
  \algshort converges to a near-optimal policy with sample complexity
  polynomial in the coverability coefficient $\Ccov(\Pi)$; in
  particular, to learn an $\veps$-optimal policy
  $T=\bigoht\prn*{\frac{\Ccov(\Pi)\log\abs{\Pi}}{\veps^2}}$ episodes
  are required.\footnote{We state the result for finite
  classes\arxiv{ ($\log\abs{\Pi}<\infty$)} to simplify presentation,
  following the standard in RL theory
  \citep{agarwal2019reinforcement,foster2023foundations}\loose\arxiv{;
  \arxiv{the result}\neurips{it}\arxiv{ readily} extends to infinite classes through standard arguments.\loose}} By scaling with $\Ccov(\Pi)$, \cref{thm:main} can be viewed as a strict improvement over
offline RLHF \citep{zhu2023principled,zhan2023provable}, as well as prior works on online RLHF that rely on passive exploration
\citep{xiong2023gibbs,gao2024rebel,chang2024dataset}. In particular,
these works scale with \emph{coverage parameters} for $\piref$, the simplest of
which take the form $\Cconc(\Pi)
    \ldef{}
    \sup_{\tau\in(\cS\times\cA)^{H}}\sup_{\pi\in\Pi}\frac{\pi(\tau)}{\piref(\tau)}$.
Under \cref{ass:vmax}, we have that $\Cconc(\Pi) =
    \exp(\Vmax/\beta)$ which, as discussed above, upper bounds
    $\Ccov(\Pi)$ but can be much larger when $\piref$ has poor
    coverage. The dependence on $\Ccov(\Pi)$ in \cref{thm:main}
reflects the fact
that \algshort can explore responses not covered by
$\piref$.\footnote{
    Many works consider more general
  notions of coverage that account
  for reward function structure, in the same vein as SEC, as well as
  single-policy variants; both can be problematic for similar reasons.} 

In
\cref{sec:proofs_main}, we give a generalization of \cref{thm:main} (\cref{thm:main_general})
which scales with a more comprehensive exploration parameter, the \emph{Sequential
  Extrapolation Coefficient} (SEC), matching (for DCMDPs) the most general results
in prior work on exploration in RLHF, but with a significantly
simpler algorithm \citep{chen2022human,wang2023rlhf,ye2024theoretical}. The SEC also leads to polynomial sample
complexity for tabular and linear MDPs, a common setting considered in
prior work
\citep{xu2020preference,novoseller2020dueling,pacchiano2021dueling,wu2023making,zhan2023query,das2024provably}. See
\cref{sec:related} for a detailed comparison. We emphasize that \cref{thm:main} applies to any DCMDP (including but not limited to the token-level MDP), even if the
    dynamics are unknown; as such, the result meaningfully extends
    beyond the \emph{contextual bandit} formulation of RLHF found in many prior works \citep{zhu2023principled,xiong2023gibbs,das2024provably,ye2024theoretical}.\loose

\begin{remark}[Nontriviality and role of $\beta$]
By avoiding explicit
dependence on $\exp(\frac{1}{\beta})$, \algshort provably improves
upon Online \dpo when $\beta$ is small; per \cref{prop:dpo_failure}, the latter must pay
$\exp(\frac{1}{\beta})$ even when $\Ccov(\Pi)\leq{}2$. This improvement stems from the fact that KL-regularization does not automatically
  lead to exploration or grant meaningful control of coverability in the
small-$\beta$ regime.\loose

To highlight the importance of the small-$\beta$ regime, we note that
by taking $\beta=\poly(1/T)$, \cref{thm:main} immediately leads to
bounds on the \emph{unregularized} reward $J(\pi)$. This would not be
possible if the sample complexity guarantee explicitly scaled with $\exp(\frac{1}{\beta})$.

\end{remark}

\arxiv{\paragraph{Computational efficiency}}
\neurips{\textbf{Computational efficiency.}~~}
Most prior approaches to RL with general function approximation that
incorporate global forms of optimism similar to \cref{eq:bonus} 
\citep{jiang2017contextual,sun2019model,du2021bilinear,jin2021bellman,xie2023role,liu2024maximize}
are known to be computationally intractable to implement in general
\citep{dann2018oracle}, and involve solving non-convex,
non-differentiable constrained optimization problems. Thus, it is natural to
ask why our result is not too
good to be true. The answer is that even though the objective in
\cref{eq:opt_dpo} is simple, it is still non-convex in general, even if one employs
log-linear policies of the form \neurips{$\pi_\theta(a\mid{}s)
  \propto\exp\prn[\big]{\frac{1}{\beta}\tri*{\phi(s,a),\theta}}$} 
\arxiv{\begin{equation}
  \label{eq:log_linear}
\pi_\theta(a\mid{}s) \propto\exp\prn*{\frac{1}{\beta}\tri*{\phi(s,a),\theta}}
\end{equation}}%
for $\theta\in\bbR^{d}$. This non-convexity is precisely caused by the
presence of the optimistic term
\cref{eq:bonus}; \cref{thm:main} is valid for all choices of $\beta>0$,
but we expect that the optimization problem in \cref{eq:opt_dpo} will
become more difficult to solve as $\beta\to{}0$.\footnote{Interestingly, one can show that for an
  appropriate \arxiv{choice of }$\alpha$, our objective converges to the
  standard global optimism objective \citep{jin2021bellman} under
  \arxiv{\cref{eq:log_linear}}\neurips{this parameterization} as $\beta\to{}0$. Conversely for
  very large $\beta$ ($\beta\approxgeq\Rmax$), the objective becomes
  convex. We leave a dedicated analysis of the optimization landscape
  for future work.\loose
} In light of this, our
work can be viewed as using the
unique structure of the KL-regularized MDP formulation and
deterministic contextual MDP (DCMDP) to derive an optimistic
exploration objective which---while
still non-convex---is differentiable and directly amenable to
a practical implementation with language models.\arxiv{\footnote{In
    particular, working in a DCMDP removes the notorious
    \emph{double-sampling} issue, which is why it suffices to estimate
    the rewards using the \dpo objective instead more complex
    objectives \citep{jiang2017contextual,jin2021bellman} that
    directly aim to minimizes Bellman errors.}}
This technique is
novel even in the context of reward-driven (as opposed to
preference-based) RL, and we expect it to find broader use.

\arxiv{\paragraph{Additional remarks}}
\neurips{\textbf{Additional remarks.}~~}
Separately, we mention in
passing that we believe it should be possible to derive tighter sample
complexity bounds for large $\beta>0$, in the vein of
\citet{tiapkin2023fast}.\loose

\begin{remark}[Limitations of the \dpo objective]
  Our results are limited to MDPs with deterministic dynamics and
  stochastic start state (DCMDPs). We believe that without further
  modifications, the \dpo objective is not suitable for stochastic
  dynamics, as \cref{eq:reward_v} no longer holds on a per-trajectory
  basis.\loose
  \end{remark}
\begin{remark}[Trajectory coverability]
  A related point concerns trajectory coverability.
  In the standard (as opposed to preference-based) RL setting, it is
  possible to achieve guarantees that scale with \emph{state-action coverability}
  \citep{xie2023role}\neurips{, defined via $    C_{\mathsf{st}}(\Pi) \ldef{}
    \inf_{\mu\in\Delta(\cS\times\cA)}\sup_{s\in\cS,a\in\cA}\sup_{\pi\in\Pi}\frac{d^{\pi}(s,a)}{\mu(s,a)}$,}\arxiv{,
    defined via:
  \[
    C_{\mathsf{st}}(\Pi) \ldef{}
    \inf_{\mu\in\Delta(\cS\times\cA)}\sup_{s\in\cS,a\in\cA}\sup_{\pi\in\Pi}\frac{d^{\pi}(s,a)}{\mu(s,a)},
  \]}
  where $d^{\pi}(s,a)\ldef{}\bbP_{\pi}\prn*{s_h=s,a_h=a}$. In general, we can have
  $C_{\mathsf{st}}(\Pi)\ll\Ccov(\Pi)$. We expect that
    trajectory-level coverability is necessary for algorithms based on the
    \dpo objective.
  Nonetheless, the difference is immaterial %
  for
  language modeling in the token-level MDP, which has
  $C_{\mathsf{st}}(\Pi)=\Ccov(\Pi)$. \loose
\end{remark}

\subsection{Proof Sketch for \creftitle{thm:main}}
\label{sec:proof_sketch}
\newcommand{\obj}{\Psi_{\texttt{XPO}}}%
Our starting point for the proof of \cref{thm:main} is the following
regret decomposition, which is proven as a consequence of the implicit
$Q^{\star}$-approximation result in \cref{eq:reward_v}.
  \begin{restatable}[Central regret decomposition]{lemma}{central}
    \label{lem:regret_decomp}
    For any pair of policies $\pi$ and $\nu$, it holds that
\neurips{    \begin{small}}
    \begin{align}
        J_\beta(\pi^\star_\beta) - J_\beta(\pi) = &~ \E_{\tau \sim
          \nu} \left[\beta\log\pi(\tau)\right] - \E_{\tau \sim \nu}
        \left[\beta\log\pi^\star_\beta(\tau)\right] \label{eq:decomp1}
        \\
        &~ + \E_{\tau \sim \pi} \left[
          \beta\log\frac{\pi(\tau)}{\piref(\tau)} - r(\tau) \right] -
        \E_{\tau\sim\nu} \left[\beta\log\frac{\pi(\tau)}{\piref(\tau)}
          - r(\tau)\right].\label{eq:decomp2}
      \end{align}%
    \neurips{\end{small}}%
  \end{restatable}
  This result decomposes the error of any policy into
  two pairs of terms: The first pair in \cref{eq:decomp1} measures the extent to
  which the policy's internal reward model overestimates the optimal
  value, and directly informs the
  notion of optimism in \algshort, while the second pair in \cref{eq:decomp2} measures the reward
  model's predictive accuracy. Critically, as a consequence of the fact that \cref{eq:reward_v}
  holds uniformly for all trajectories, the regret
  decomposition measures error under (i) the policy $\pi$ itself
  (on-policy error), and (ii) an \emph{arbitrary} reference policy $\nu$,
  which we will instantiate as the historical data distribution.

  Let
  $\bmu\ind{t}\ldef{}\frac{1}{t-1}\sum_{i<t}\pi\ind{i}\otimes\piref$
  denote the policy that, given $s_1$, samples $\tau\sim{}\pi\ind{i}$
  for $i\sim\unif(\brk{t-1})$ and samples $\tautil\sim\piref$, with
  the convention that $\bmu\ind{1}$ is arbitrary. Observe that
  $\min_{t\in\brk{T+1}}\Jb(\pistarb) - \Jb(\pi\ind{t})
      \leq{} \frac{1}{T}\sum_{t=1}^{T}\Jb(\pistarb) - \Jb(\pi\ind{t})$.
        For each step $t$, applying \cref{lem:regret_decomp} with
        $\pi=\pi\ind{t}$ and $\nu=\piref$ gives\loose
    \neurips{    \begin{small}}
      \begin{flalign}
        &\frac{1}{T}\sum_{t=1}^{T}\Jb(\pistarb) - \Jb(\pi\ind{t})
        \leq\frac{1}{T}\sum_{t=1}^{T}\E_{\tau \sim \piref}
        \left[\beta\log\pi\ind{t}(\tau)-\beta\log\pi^\star_\beta(\tau)\right]\notag\\
        &~~~~+ \frac{1}{T}\sum_{t=1}^{T}\E_{s_1\sim\rho,\tau \sim
          \pi\ind{t}\mid{}s_1,\taut\sim\piref\mid{}s_1} \left[
          \beta\log\frac{\pi\ind{t}(\tau)}{\piref(\tau)} -
          r(\tau)-\beta\log\frac{\pi\ind{t}(\taut)}{\piref(\taut)} +
          r(\taut)\right].\label{eq:sketch0}
      \end{flalign}%
\neurips{    \end{small}}%
    The reward estimation error term in \cref{eq:sketch0} samples
    $\tau\sim{}\pi\ind{t}\mid{}s_1$ and
    $\tautil\sim{}\piref\sim{}s_1$ (on-policy). To relate this to
    the purely off-policy objective in \cref{line:opt_dpo}
    of \algshort, we use a potential argument based on coverability
    \citep{xie2023role} which, for any $\alpha>0$, allows us to bound the above expression
    by
\neurips{    \begin{small}}
      \begin{align}
        &\approxleq{} \frac{\alpha}{\beta}\cdot{}\Ccov(\Pi) + \frac{1}{T}\sum_{t=1}^{T}\E_{\tau \sim \piref}
          \left[\beta\log\pi\ind{t}(\tau)-\beta\log\pi^\star_\beta(\tau)\right]\notag\\
        &~~~~+ \frac{\alpha^{-1}\beta}{T}\sum_{t=1}^{T}\E_{s_1\sim\rho,(\tau, \taut)\sim\bmu\ind{t}\mid{}s_1} \left[\prn*{
          \beta\log\frac{\pi\ind{t}(\tau)}{\piref(\tau)} -
          r(\tau)-\beta\log\frac{\pi\ind{t}(\taut)}{\piref(\taut)} +
          r(\taut)}^2\right].\label{eq:sketch1}
      \end{align}%
\neurips{    \end{small}}%
    Let
    \neurips{$\Psi_{\texttt{XPO}}\ind{t}(\pi) \ldef \E_{\tau \sim \piref}
                                          \brk[\big]{\beta\log\pi(\tau)-\beta\log\pi^\star_\beta(\tau)}+\alpha^{-1}\beta\E_{s_1\sim\rho,(\tau,\taut)\sim\bmu\ind{t}\mid{}s_1} \brk[\big]{\prn[\big]{
                                          \beta\log\frac{\pi(\tau)}{\piref(\tau)} -
                                          r(\tau)-\beta\log\frac{\pi(\taut)}{\piref(\taut)} +
                                          r(\taut)}^2}$.}
                                    \arxiv{
      \begin{align*}
        \Psi_{\texttt{XPO}}\ind{t}(\pi) &= \E_{\tau \sim \piref}
                                          \left[\beta\log\pi(\tau)-\beta\log\pi^\star_\beta(\tau)\right]\\
                                        &~~~~+\alpha^{-1}\beta\E_{s_1\sim\rho,(\tau,\taut)\sim\bmu\ind{t}\mid{}s_1} \left[\prn*{
                                          \beta\log\frac{\pi(\tau)}{\piref(\tau)} -
                                          r(\tau)-\beta\log\frac{\pi(\taut)}{\piref(\taut)} +
                                          r(\taut)}^2\right].
      \end{align*}}
           If we could choose
           $\pi\ind{t}=\argmin_{\pi\in\Pi}\obj\ind{t}(\pi)$, we would be
           done, since by \cref{eq:reward_v} this would yield
\neurips{           \begin{small}}
             \begin{align*}
               \obj\ind{t}(\pi\ind{t})
               \leq{}              \obj\ind{t}(\pistarb)
               = \E_{s_1\sim\rho,(\tau,\taut)\sim\bmu\ind{t}\mid{}s_1}
               \left[\prn*{ \beta\log\frac{\pistarb(\tau)}{\piref(\tau)} -
               r(\tau)-\beta\log\frac{\pistarb(\taut)}{\piref(\taut)} +
               r(\taut)}^2\right] = 0.
             \end{align*}%
\neurips{           \end{small}}%
          The \algshort objective in
           \cref{line:opt_dpo} minimizes an
           empirical analogue of this quantity (up to a standard translation
           between log-loss and square loss under the Bradley-Terry model), so a concentration
           argument (\cref{lem:conc_main}) allows us to conclude that
           the iterates of \algshort satisfy
           $\obj\ind{t}(\pi\ind{t})\approxleq{}\alpha^{-1}\frac{\log\abs{\Pi}}{t}
           + \sqrt{\frac{\log\abs{\Pi}}{t}}$\arxiv{
           with high probability}.
           Plugging this bound into
           \cref{eq:sketch1} yields \neurips{$             \frac{1}{T}\sum_{t=1}^{T}\Jb(\pistarb) - \Jb(\pi\ind{t})
             \approxleq{}              \sqrt{\frac{\Ccov(\Pi)\log\abs{\Pi}}{T}}$}
           \arxiv{\begin{align*}
             \frac{1}{T}\sum_{t=1}^{T}\Jb(\pistarb) - \Jb(\pi\ind{t})
             \approxleq{}              \sqrt{\frac{\Ccov(\Pi)\log\abs{\Pi}}{T}}
           \end{align*}}
           after tuning $\alpha$.\loose

\subsection{Empirical Validation}
\label{sec:experiments}

\newcommand{\alpaca}{AlpacaEval-2\xspace}
\newcommand{\llamasft}{Llama-3-8B-Flow-SFT\xspace}
\newcommand{\pairrm}{\texttt{PairRM}\xspace}
\newcommand{\llamafinal}{Llama-3-8B-Flow-Final\xspace}

To close this section, we provide a preliminary empirical evaluation of \algshort in real-world RLHF experiments.
To implement \algshort, we use the iterative \dpo
\citep{xu2023some,snorkelai2024,dong2024rlhf} pipeline from
  \citet{dong2024rlhf} with 3
total iterations (that is, we set $T=3$, but draw a large batch of
pairs from $\pi\ind{t}$), and augment the \dpo objective with the optimism term in \algshort.
We use the same base model (which we refer to as \llamasft),\footnote{\scriptsize
  \url{https://huggingface.co/RLHFlow/LLaMA3-SFT}.} prompt sets for
each iteration,\footnote{\scriptsize
  \url{https://huggingface.co/datasets/RLHFlow/iterative-prompt-v1-iter1-20K},
  \url{https://huggingface.co/datasets/RLHFlow/iterative-prompt-v1-iter2-20K},
  \url{https://huggingface.co/datasets/RLHFlow/iterative-prompt-v1-iter3-20K}.}
and preference model (to generate preference feedback),\footnote{\scriptsize \url{https://huggingface.co/RLHFlow/pair-preference-model-LLaMA3-8B}.} as \citet{dong2024rlhf}, which makes our results generally comparable to theirs.
Over all three iterations, we fix $\piref$ to be the base model, \llamasft.

\neurips{\vspace{-5pt}}
\begin{table}[ht!]
\centering
\begin{adjustbox}{max width=\textwidth}
\begin{tabular}{c|ccccc|cc}
\textbf{Model}             & \textbf{AGIEval} & \textbf{ANLI}  & \textbf{MMLU}  & \textbf{GPQA}  & \textbf{GSM8K} & \textbf{\alpaca LC} & \textbf{Arena-Hard} \\ \hline
Llama-3-8B-Flow-SFT   & 43.71            & 40.16          & 62.48          & 30.37          & 73.46          & 9.08                & 9.4                 \\ \hline
\dpo-iter1             & 46.38            & 46             & 63.01          & 29.99          & \textbf{77.79} & 23.27               & 19.2                \\
\dpo-iter2             & 47.47            & 48.13          & 63.45          & 29.32          & 72.1           & 23.68               & 20.1                \\
\dpo-iter3             & 46.92            & 47.72 & 63.55          & 28.73          & 72.93          & 27.33               & 24.9                \\ \hline
\algshort-iter1             & 46.73            & 45.47          & 63.1           & \textbf{30.41} & 76.57          & 22.14               & 18.7                \\
\algshort-iter2             & 47.44            & \textbf{48.44}          & 63.35          & 29.95          & 75.51          & 24.65               & 23.2                \\
\algshort-iter3             & \textbf{47.48}   & 48.09          & \textbf{63.72} & 29.49          & 75.97          & \textbf{29.35}      & \textbf{27.2}       \\ \hline\hline
Llama-3-8B-Flow-Final & 47.29            & 46.12          & 63.37          & 29.45          & \underline{78.62}    & \underline{31.7}                & 23.1                \\
Llama-3-8B-it         & 44.78            & 46.47          & \underline{63.79}    & \underline{30.83}    & 75.59          & \underline{32.58}         & 20.7                
\end{tabular}
\end{adjustbox}
\vspace{2pt}
\caption{Benchmarks for \algshort and baseline
  models. \textbf{Bold} indicates best performance with the
  same data usage. \underline{Underlined} results are superior to
  \algshort, but from a model requiring more data or industry-level.\loose}
\label{tb:llama_main}
\end{table}
\neurips{\vspace{-10pt}}

In \cref{tb:llama_main}, we compare \algshort with the following
baselines: 1) iterative \dpo with the same setup (i.e., \algshort with
$\alpha=0$), 2) \llamafinal, the final model from
\citet{dong2024rlhf}, and 3) the industry-level instruction-tuned model
Llama-3-8B-it,\footnote{\scriptsize\url{https://huggingface.co/meta-llama/Meta-Llama-3-8B-Instruct}}
on various academic and chat benchmarks
\citep{zhong2023agieval,nie2019adversarial,hendrycks2021measuring,rein2023gpqa,cobbe2021training,dubois2024length,arenahard2024,clark2018think,lin2021truthfulqa,zellers2019hellaswag,sakaguchi2021winogrande}. We
compute all
baseline numbers ourselves with the same configuration for a fair comparison. A key distinction between our experimental setup and that of
\citet{dong2024rlhf} is that we construct preference pairs \emph{from only two responses}, whereas \citet{dong2024rlhf} use
best/worst-over-8-responses for preference pair construction as a
heuristic exploration strategy. In other words, the final models we
obtain (\algshort-iter3, and a baseline, \dpo-iter3 in
\cref{tb:llama_main}) \emph{use only $\nicefrac{1}{4}$ the number of generated responses} compared the final model (\llamafinal)\footnote{\scriptsize\url{https://huggingface.co/RLHFlow/LLaMA3-iterative-DPO-final}} from \citet{dong2024rlhf}.\loose

We find that the model obtained by \algshort improves
over the non-exploratory baseline (\dpo-iter) on the chat benchmarks (which offer roughly $\sim$90\% agreement and/or Spearman
correlation to Chatbot Arena \citep{chiang2024chatbot}), and attains performance comparable to the industry-level (Llama-3-8B-it) or
4$\times$data-usage (\llamafinal) models.
At the same time, \algshort also improves over the non-exploratory
baseline on most of the academic benchmarks, again achieving comparable performance with the industry-level and
4$\times$data-usage models, and does not introduce significant
performance regression in any benchmark. In contrast, we observe that the
iterative \dpo baseline (without exploration) causes obvious regression in the math (GSM8K)
benchmark.
However, we caution that conducting separate training runs with different random seeds can yield results with relatively high variance (e.g.,
the difference in win rates can be up to $3\%\sim4\%$) for both chat
benchmarks; due to resource limitations, we defer a more comprehensive
evaluation to future work. See \cref{sec:exp_details} for additional results.

\neurips{\vspace{-5pt}}
\section{Discussion}
\label{sec:discussion}

Our work provides the first practical and provably
        sample-efficient online exploration algorithm for RLHF with
        general function approximation, a step toward fully realizing the potential of online
        exploration for aligning language models. Our results also show that viewing \dpo as a form of implicit
$Q^{\star}$-approximation can directly
inform new algorithmic interventions (e.g., implicit optimism), and offer an example of fruitful
interplay between language modeling and theoretical reinforcement
learning. Building on this viewpoint, an exciting direction for future
work is to import the broader set of tools from the literature on
reinforcement learning theory (e.g., more powerful exploration
principles \citep{foster2021statistical}) and harness them for
language modeling and alignment; in this context, we expect our analysis techniques based on
the KL-regularized MDP to find broader use.

From a reinforcement learning perspective, interesting technical
directions for future work include (i) providing instance-dependent
sample complexity bounds for \algshort; and (ii) supporting RL
settings beyond deterministic contextual MDPs. On the practical side, immediate followup directions include extending
\algshort to support general preference models
\citep{munos2023nash,swamy2024minimaximalist} or more general feedback
modalities \citep{ethayarajh2024kto}.\loose

\clearpage

\bibliography{refs}

\clearpage

\appendix

\renewcommand{\contentsname}{Contents of Appendix}
\addtocontents{toc}{\protect\setcounter{tocdepth}{2}}
{
  \hypersetup{hidelinks}
  \tableofcontents
}

\section{Related Work}
\label{sec:related}

\paragraph{Theoretical algorithms for RLHF}

Theoretical analysis of algorithms for RLHF is becoming an active area
of research. Much of this research focuses on purely offline RLHF
\citep{zhu2023principled,zhan2023provable}, which is complementary to
our work. Many works also consider a so-called \emph{hybrid} RLHF setting,
where the algorithm has access to online feedback, but requires the
initial policy $\piref$ to have good coverage (e.g., bounded
concentrability or related quantities)
\citep{xiong2023gibbs,gao2024rebel,chang2024dataset}.\footnote{To our
  knowledge, all prior works in this space require uniform
  notions of concentrability as opposed to single-policy
  concentrability. \citet{gao2024rebel} state guarantees in terms of
  single-policy concentrability under the assumption that certain regression
  errors can be bounded, but this cannot be achieved in general
  without further coverage or exploration-like conditions.}
  These hybrid
algorithms do not engage in systematic exploration (i.e., they explore
passively), and hence cannot provide meaningful guarantees if $\piref$
does not adequately cover the optimal policy (e.g., for the setting in \cref{prop:dpo_failure}). 

For online RLHF, the most relevant related work can be summarized as follows:
\begin{itemize}
\item
Most prior work \citep{xu2020preference,novoseller2020dueling,pacchiano2021dueling,wu2023making,zhan2023query,du2024exploration,das2024provably}
  gives algorithms and sample complexity guarantees for the special
  case of tabular or linear MDPs; these algorithms use
  exploration bonuses that are tailored to linear models, and are not
  suitable for the general function approximation setting we consider
  (e.g., for LLMs). Nonetheless, we obtain polynomial sample complexity guarantees
  for tabular and linear MDPs (\cref{ex:tabular,ex:linear}), though
  our results are restricted to deterministic dynamics (we believe
  that moving beyond the \dpo objective is likely required to handle
  stochastic dynamics).
\item More relevant to our work is
  \citet{ye2024theoretical}, who give algorithms and sample complexity
  guarantees for online RLHF with general function
  approximation for the special case of contextual bandits
  ($H=1$). For contextual bandits, their sample complexity guarantees scale with a
  complexity measure, the \emph{eluder coefficient}, which is
  equivalent to the Sequential Extrapolation Coefficient in our most
  general result, \cref{thm:main_general}. However, their exploration
  algorithm requires solving a rather complicated optimization
  problem, and it is unclear whether it is possible to implement it
  faithfully for language models (in particular, their experiments use
  an alternative, heuristic approach to exploration which is only
  loosely inspired by the theory).
\item Lastly, \citet{chen2022human,wang2023rlhf} give guarantees for
  RLHF with general function approximation based on eluder
  dimension-like complexity measures which are incomparable to, but in some cases more
  general than \cref{thm:main_general}. However, these works require
  model-based function approximation (as opposed to the model-free
  setup we consider), and do not lead to efficient or practical
  algorithms when specialized to language modeling.
\end{itemize}

A difference worth highlighting between our work and some (but not all) of the
works above \citep{zhu2023principled,xiong2023gibbs,das2024provably,ye2024theoretical} is that we model RLHF as a general reinforcement learning
problem as opposed to a contextual bandit problem. The
problem of autoregressive sequence prediction can equivalently be
formulated as RL in the token-level MDP, or as a contextual bandit
problem (RL with horizon $H=1$) in which the ``action space'' consists
of all possible token sequences. However, because our work supports general
deterministic contextual MDPs (DCMDPs) with unknown dynamics and not
just the token-level MDP, it is strictly more general than the contextual bandit formulation.

Recent work of \citet{rafailov2024r} shows that \dpo, when applied to
the token-level MDP can be viewed as estimating the KL-regularized
value function $\Qstarb$; their work does not consider sample
complexity or online exploration. Our results extend their observation to any
deterministic contextual MDP and---more importantly---show that it is
possible to harness this perspective to provide provable end-to-end
sample complexity guarantees.

\paragraph{Empirical algorithms for exploration in RLHF}
Online exploration in RLHF has received limited exploration so far,
with notable examples including Online \dpo \citep{guo2024direct} and
Iterative \dpo \citep{xu2023some,snorkelai2024,pang2024iterative,mitra2024orca,dong2024rlhf}. As
discussed in \cref{sec:background}, these methods engage in purely
\emph{passive} exploration, meaning that sample from the current model
$\pi\ind{t}$ without an explicit mechanism to encourage diverse,
exploratory responses.

\citet{dwaracherla2024efficient} perform a dedicated empirical
evaluation of active exploration for language models. However, this work does not actually
train the language model, and thus cannot be viewed as a form
of RLHF; instead the authors train a reward model iteratively, and use
this in tandem with various active sampling schemes to accept or reject
responses proposed by $\piref$. Nevertheless, the positive results
achieved by \citet{dwaracherla2024efficient} in this limited setting
are suggestive of the potential power of online exploration in RLHF. Similarly, \citet{ye2024theoretical} perform a limited evaluation of empirical exploration schemes inspired by theoretical RL, but only report results for reward modeling benchmarks, not language modeling.

Most closely related, \citet{xiong2023gibbs,dong2024rlhf} perform an
extensive empirical evaluation of Iterative \dpo variants, and find
that Iterative \dpo with passive exploration can already have
significant benefits over offline \dpo. These works also incorporate a
``best/worst-over-$n$'' trick for preference
pair construction, which can be viewed as a heuristic to
promote exploration, but does not have provable guarantees. See
\cref{sec:experiments} for further discussion.

\paragraph{Theoretical reinforcement learning}
Outside the context of language models, an active line of research
provides structural complexity measures and algorithms that
enable sample-efficient exploration in reinforcement learning in
general settings
\citep{russo2013eluder,jiang2017contextual,sun2019model,wang2020provably,du2021bilinear,jin2021bellman,foster2021statistical,xie2023role,foster2023tight,liu2024maximize}. The
techniques from this line of research that support general function
approximation, while sample-efficient, are computationally intractable
to implement in general \citep{dann2018oracle}, involving non-convex
and non-differentiable constrained optimization problems. We use the
unique structure of the KL-regularized MDP formulation and
deterministic contextual MDP (DCMDP) to derive the exploration
objective in \algshort which---while
still non-convex---is differentiable and directly amenable to
a practical implementation with language models.

\paragraph{Entropy- and KL-regularized reinforcement learning}
First introduced in \citet{ziebart2008maximum,ziebart2010modeling}, a
number of recent works provide sample complexity guarantees for
reinforcement learning in KL-regularized or entropy-regularized MDPs
\citep{kozuno2022kl,tiapkin2023regularized,tiapkin2023fast}, mainly focusing on the special case of
tabular (finite-state/action) MDPs. To the best of our knowledge, the
optimistic objective in \algshort is novel in this context.

\section{Technical Tools}
\label{sec:technical}

\begin{lemma}[Azuma-Hoeffding]
  \label{lem:hoeffding}
    Let $(X_t)_{t\leq{T}}$ be a sequence of real-valued random
    variables adapted to a filtration $\prn{\filt_t}_{t\leq{}T}$. If $\abs*{X_t}\leq{}R$ almost surely, then with probability at least $1-\delta$,
    \[
      \abs*{\sum_{t=1}^{T}X_t - \En_{t-1}\brk*{X_t}} \leq{} R\cdot\sqrt{8T\log(2\delta^{-1})}.
    \]
\end{lemma}

\begin{lemma}[Martingale Chernoff (e.g., \citealp{foster2021statistical})]
  \label{lem:martingale_chernoff}
  For any sequence of real-valued random variables $\prn{X_t}_{t\leq{}T}$ adapted to a
  filtration $\prn{\filt_t}_{t\leq{}T}$, it holds that with probability at least
  $1-\delta$, for all $T'\leq{}T$,
  \begin{equation}
    \label{eq:martingale_chernoff}
    \sum_{t=1}^{T'}-\log\prn*{\En_{t-1}\brk*{e^{-X_t}}}      \leq \sum_{t=1}^{T'}X_t
    + \log(\delta^{-1}).
  \end{equation}
\end{lemma}

\section{Proof of \creftitle{thm:main}}
\label{sec:proofs_main}

This section is organized as follows. First, in
\cref{sec:main_general}, we present a more general version of
\algshort, which makes use of an arbitrary, user-specified sampling
policy for the second response $\tautil$. Then, in
\cref{sec:main_general}, we state a more general version of
\cref{thm:main} (\cref{thm:main_general}), and show how it implies
\cref{thm:main}. Examples are then given in \cref{sec:main_examples}.

In the remainder of the section, we prove \cref{thm:main_general}.
We first prove a number of intermediate results:
\begin{itemize}
\item In \cref{sec:regularized}, we state preliminaries regarding the
  KL-regularized MDP, and use them to prove the implicit
  $Q^{\star}$-approximation lemma (\cref{lem:implicit_q}).
\item In \cref{sec:regret_decomp}, we prove the central regret decomposition
  lemma (\cref{lem:regret_decomp}).
\item In \cref{sec:conc}, we prove a key concentration result
  used within \cref{thm:main_general}.
\end{itemize}
Finally, in
\cref{sec:main_proof}, we prove \cref{thm:main_general}, with proofs
for supporting lemmas deferred to \cref{sec:main_supporting}.

\subsection{General Version of \algshort}
\label{sec:general_alg}

\begin{algorithm}[ht]
\caption{\mainalg (\algshort) with general sampling policy.}
\label{alg:general}
\begin{adjustbox}{max width=\textwidth}
\begin{minipage}{\linewidth}
\begin{algorithmic}[1]
  \Statex[0] \multiline{{\bfseries input:}
    Number of iterations $T$, KL-regularization coefficient $\beta>0$,
    optimism coefficient $\alpha>0$, sampling strategy $\samp$.}%
\State Initialize $\pi^\iter{1},\pisamp\ind{1} \leftarrow \piref$, $\cD_\pref\ind{0}\gets\emptyset$.
\For{iteration $t = 1,2,\dotsc,T$}
    \State \textbf{Generate response pair
      $(\tau^\iter{t},\taut^\iter{t})$ via:} $s\ind{t}_1 \sim \rho$, $\tau^\iter{t} \sim \pi^\iter{t} \mid s\ind{t}_1$, and $\taut^\iter{t} \sim \pisamp^\iter{t} \mid s\ind{t}_1$.
    \State \textbf{Label with preference:} Label $(\tau^\iter{t},\taut^\iter{t})$ as
      $(\tau^\iter{t}_+,\tau^\iter{t}_-)$ with preference $y\ind{t}\sim{}\Pr(\tau^\iter{t} \succ \taut^\iter{t})$.
    \State \textbf{Update preference data:} $\Dcal_\pref^\iter{t} \leftarrow \Dcal_\pref^\iter{t-1} \bigcup \{(\tau^\iter{t}_+,\tau^\iter{t}_-)\}$.
    \State \multiline{\textbf{Update optimism data:} Compute dataset $\cDopt\ind{t}$ of $t$ samples
      from $\pisamp\ind{t}$.}
    \Statex[0]\hfill\algcommentlight{When
      $\pisamp\ind{t}=\piref$, can re-use previous samples as in \cref{alg:opt_dpo}.}
    \State \label{line:opt_dpo_general}%
    \textbf{Direct preference optimization with global optimism:}
    Calculate $\pi^\iter{t+1}$ via
    \begin{small}
      \begin{align*}
        \pi^\iter{t+1} \leftarrow \argmin_{\pi \in \Pi}\crl*{ \alpha\sum_{\tau\in\cDopt\ind{t}}\log\pi(\tau) -\sum_{(\tau_+,\tau_-) \in \Dcal_\pref^\iter{t}} \log\left[\sigma\left( \beta\log\frac{\pi(\tau_+)}{\piref(\tau_+)} - \beta\log\frac{\pi(\tau_-)}{\piref(\tau_-)} \right) \right]}.
      \end{align*}
    \end{small}
    \State \textbf{Update sampling policy:}
    $\pisamp\ind{t+1}\gets\samp(\pi\ind{1},\ldots,\pi\ind{t+1})$.
    \EndFor
    \State \textbf{return:}
    $\pihat=\argmax_{\pi\in\crl{\pi\ind{1},\ldots,\pi\ind{T+1}}}J_\beta(\pi\ind{t})$.\hfill\algcommentlight{Can
      compute using validation data.}
\end{algorithmic}
\end{minipage}
\end{adjustbox}
\end{algorithm}

\cref{alg:general} presents a general version of \algshort. The
algorithm is identical to \cref{alg:opt_dpo}, except that it makes use
of an arbitrary, user-specified user-specified sampling
policy for the second response $\tautil$.

In more detail, the algorithm takes as input a \emph{sampling
  strategy} $\samp$ which, at step $t$, computes a \emph{sampling
  policy} $\pitil\ind{t}$ via
$\pitil\ind{t}\gets\samp(\pi\ind{1},\ldots,\pi\ind{T})$. The algorithm
then samples the response pair $(\tau\ind{t},\tautil\ind{t})$ via
$\tau\ind{t}\sim\pi\ind{t}\mid{}s_1\ind{t}$ and
$\tautil\ind{t}\sim\pisamp\ind{t}\mid{}s_1\ind{t}$. \cref{alg:opt_dpo}
is a special case of this scheme in which $\pisamp\ind{t}=\piref$ for
all $t$.

A secondary difference from \cref{alg:opt_dpo} is that
\cref{alg:general} assumes access to a dataset $\cDopt\ind{t}$
consisting of $t$ responses sampled from $\pitil\ind{t}$, which are
used to compute the optimistic term in \cref{line:opt_dpo_general}. In
\cref{alg:opt_dpo}, because $\pitil=\piref$ is static, we can simply re-use the responses
$\tautil\ind{1},\ldots,\tautil\ind{t}$ for this task, setting
$\cDopt\ind{t}=\crl*{\tautil\ind{1},\ldots,\tautil\ind{t}}$. However,
for general time-varying sampling scheme, it may be necessary to draw
a fresh dataset of responses from $\pitil\ind{t}$ to compute $\cDopt\ind{t}$.

As a practical example, \cref{alg:mu}---displayed below---instantiates the general scheme
in \cref{alg:general} by setting
$\pitil\ind{t}=\unif(\pi\ind{1},\ldots,\pi\ind{t})$ to sample from the
historical data distribution at step $t$. For this scheme, it suffices
to set $\cDopt\ind{t}=\crl*{\tau\ind{1},\ldots,\tau\ind{t}}$, re-using
the responses sampled from $\pi\ind{1},\ldots,\pi\ind{t}$.

\begin{algorithm}[ht]
\caption{\mainalg (\algshort) with historical sampling.}
\label{alg:mu}
\begin{adjustbox}{max width=\textwidth}
\begin{minipage}{\linewidth}
\begin{algorithmic}[1]
  \Statex[0] \multiline{{\bfseries input:}
    Number of iterations $T$, KL-regularization coefficient $\beta>0$,
    optimism coefficient $\alpha>0$, sampling strategy $\samp$.}%
\State Initialize $\pi^\iter{1},\pisamp\ind{1} \leftarrow \piref$, $\cD_\pref\ind{0}\gets\emptyset$.
\For{iteration $t = 1,2,\dotsc,T$}
    \State \textbf{Generate response pair
      $(\tau^\iter{t},\taut^\iter{t})$ via:} $s\ind{t}_1 \sim \rho$, $\tau^\iter{t} \sim \pi^\iter{t} \mid s\ind{t}_1$, and $\taut^\iter{t} \sim \unif(\pi\ind{1},\ldots,\pi\ind{t})\mid s\ind{t}_1$.
    \State \textbf{Label with preference:} Label $(\tau^\iter{t},\taut^\iter{t})$ as
      $(\tau^\iter{t}_+,\tau^\iter{t}_-)$ with preference $y\ind{t}\sim{}\Pr(\tau^\iter{t} \succ \taut^\iter{t})$.
    \State \textbf{Update preference data:} $\Dcal_\pref^\iter{t} \leftarrow \Dcal_\pref^\iter{t-1} \bigcup \{(\tau^\iter{t}_+,\tau^\iter{t}_-)\}$.
    \State \multiline{\textbf{Update optimism data:} Compute dataset $\cDopt\ind{t}$ of $t$ samples
      from $\pisamp\ind{t}$.}
    \Statex[0]\hfill\algcommentlight{When
      $\pisamp\ind{t}=\piref$, can re-use previous samples as in \cref{alg:opt_dpo}.}
    \State \label{line:opt_dpo_mu}%
    \textbf{Direct preference optimization with global optimism:}
    Calculate $\pi^\iter{t+1}$ via
    \begin{small}
      \begin{align*}
        \pi^\iter{t+1} \leftarrow \argmin_{\pi \in \Pi}\crl*{ \alpha\sum_{i=1}^{t}\log\pi(\tau\ind{i}) -\sum_{(\tau_+,\tau_-) \in \Dcal_\pref^\iter{t}} \log\left[\sigma\left( \beta\log\frac{\pi(\tau_+)}{\piref(\tau_+)} - \beta\log\frac{\pi(\tau_-)}{\piref(\tau_-)} \right) \right]}.
      \end{align*}
    \end{small}
    \EndFor
    \State \textbf{return:}
    $\pihat=\argmax_{\pi\in\crl{\pi\ind{1},\ldots,\pi\ind{T+1}}}J_\beta(\pi\ind{t})$.\hfill\algcommentlight{Can
      compute using validation data.}
\end{algorithmic}
\end{minipage}
\end{adjustbox}
\end{algorithm}

\subsection{General Version of \creftitle{thm:main}}
\label{sec:main_general}

Our most general sample complexity guarantee for \algshort
(\cref{alg:opt_dpo} and \cref{alg:general}),
\cref{thm:main_general}, is stated in terms of the following
preference-based analogue of the \emph{Sequential Extrapolation
  Coefficient} (SEC) from \citet{xie2023role} (also known as an eluder
coefficient or decoupling coefficient
\citep{zhong2022gec,ye2024theoretical}). Recall that for a trajectory
$\tau=(s_1,a_1),\ldots,(s_H,a_H)$, we define
\begin{equation}
  \pi(\tau)=\prod_{h=1}^{H}\pi(a_h\mid{}s_h),\quad\textrm{and}\quad{}r(\tau)=\sum_{h=1}^{H}r(s_h,a_h).
\end{equation}
For a pair of policies $\pi$ and $\pitil$, we define $\pi\otimes\pitil$ as
the joint policy that, given $s_1$, samples $\tau\sim{}\pi\mid{}s_1$
and $\tautil\sim\pitil\mid{}s_1$. We write
$(\tau,\tautil)\sim{}\pi\otimes\pitil\mid{}s_1$ as shorthand for this process.

\begin{definition}[Sequential Extrapolation Coefficient]
  For a policy class $\Pi$, sampling strategy $\samp$, and entropy regularization parameter
  $\beta>0$, we define the Sequential Extrapolation Coefficient via
  \begin{small}
    \begin{align}
      \label{eq:sec}
      \coefffull
      = \sup_{\pi\ind{1},\ldots,\pi\ind{T}\in\Pi}\crl*{
      \sum_{t=1}^{T}\frac{\prn*{\E_{s_1 \sim \rho, \tau \sim
      \pi^\iter{t}\mid s_1, \taut \sim \pisamp\ind{t-1} \mid s_1} \left[
      \beta\log\frac{\pi^\iter{t}(\tau)}{\piref(\tau)} - r(\tau) -
      \beta\log\frac{\pi^\iter{t}(\taut)}{\piref(\taut)} + r(\taut)
      \right]}^2}{\Vmax^2\vee
      (t-1)\cdot{}\E_{s_1 \sim \rho, (\tau,\taut) \sim \bmu\ind{t} \mid s_1} \left[\prn*{ \beta\log\frac{\pi^\iter{t}(\tau)}{\piref(\tau)} - r(\tau) - \beta\log\frac{\pi^\iter{t}(\taut)}{\piref(\taut)} + r(\taut) }^2 \right]} 
      },
    \end{align}%
      \end{small}%
      where $\pitil\ind{t}=\samp(\pi\ind{1},\ldots,\pi\ind{t})$, and where
      we define
      $\bmu\ind{t}\ldef{}\frac{1}{t-1}\sum_{i<t}\pi\ind{i}\otimes\pitil\ind{i}$, with the
      convention that $\mu\ind{1}$ is arbitrary.
    \end{definition}

Note that for \cref{alg:opt_dpo}, which sets $\pitil\ind{t}=\piref$
for all $t$, we can simplify the definition above to
  \begin{small}
    \begin{align}
      \label{eq:sec_ref}
      \coeff(\Pi,T,\beta;\piref)
      \ldef \sup_{\pi\ind{1},\ldots,\pi\ind{T}\in\Pi}\crl*{
      \sum_{t=1}^{T}\frac{\prn*{\E_{s_1 \sim \rho, \tau \sim
      \pi^\iter{t}\mid s_1, \taut \sim \piref \mid s_1} \left[
      \beta\log\frac{\pi^\iter{t}(\tau)}{\piref(\tau)} - r(\tau) -
      \beta\log\frac{\pi^\iter{t}(\taut)}{\piref(\taut)} + r(\taut)
      \right]}^2}{\Vmax^2\vee
      (t-1)\cdot{}\E_{s_1 \sim \rho, (\tau,\taut) \sim \bmu\ind{t} \mid s_1} \left[\prn*{ \beta\log\frac{\pi^\iter{t}(\tau)}{\piref(\tau)} - r(\tau) - \beta\log\frac{\pi^\iter{t}(\taut)}{\piref(\taut)} + r(\taut) }^2 \right]} 
      },
    \end{align}%
  \end{small}%
  where $\bmu\ind{t}\ldef{}\frac{1}{t-1}\sum_{i<t}\pi\ind{i}\otimes\piref$.
    
\paragraph{Main sample complexity guarantee}
    
Our general sample complexity guarantee is as follows.
\begin{thmmod}{thm:main}{$'$}[General version of \cref{thm:main}]
  \label{thm:main_general}
  Suppose \cref{ass:realizability,ass:vmax} hold.
  For any $\beta>0$ and $T\in\bbN$, if we set $\alpha
=c\cdot{}\frac{\beta}{(\Vmax+\Rmax)e^{2\Rmax})}\cdot\sqrt{\frac{\log(\abs{\Pi}T\delta^{-1})\log(T)}{T\cdot{}\coefffull}}$
  for an absolute constant $c>0$,
then \cref{alg:general} ensures that
  with probability at least $1-\delta$,
  \begin{align*}
    \Jb(\pistarb) - \Jb(\pihat)
    \approxleq
(\Vmax+\Rmax)e^{2\Rmax}\cdot\sqrt{\frac{\coefffull\log(\abs{\Pi}T\delta^{-1})\log(T)}{T}}.
  \end{align*}
  As a special case, if we set $\alpha
  =c\cdot{}\frac{\beta}{(\Vmax+\Rmax)e^{2\Rmax}}\cdot\sqrt{\frac{\log(\abs{\Pi}T\delta^{-1})\log(T)}{T\cdot{}\coeff(\Pi,T,\beta;\piref)}}$
  for an absolute constant $c>0$,
then \cref{alg:opt_dpo} ensures that
  with probability at least $1-\delta$,
  \begin{align*}
    \Jb(\pistarb) - \Jb(\pihat)
    \approxleq
(\Vmax+\Rmax)e^{2\Rmax}\cdot\sqrt{\frac{\coeff(\Pi,T,\beta;\piref)\log(\abs{\Pi}T\delta^{-1})\log(T)}{T}}.
  \end{align*}

\end{thmmod}

The following result shows that the SEC is always bounded by the
coverability coefficient in \cref{def:coverability}.
\begin{lemma}
  \label{lem:sec_coverability}
  Suppose that $\samp$ sets $\pitil\ind{t}=\pitil$ for an arbitrary
  fixed policy $\pitil$ (e.g., $\pitil=\piref$). Then for any policy class $\Pi$ and $\beta>0$, it holds that for all $T\in\bbN$,
\begin{align}
  \label{eq:sec_cov}
\coefffull\leq \bigoh\prn*{\Ccov(\Pi)\cdot{}\log(T)}.
\end{align}
\end{lemma}

\cref{thm:main} follows immediately by combining
\cref{thm:main_general} with \cref{lem:sec_coverability}.

\subsection{Additional Examples for \creftitle{thm:main_general}}
\label{sec:main_examples}

In this section, we apply \cref{thm:main_general} and bound the SEC
for \emph{log-linear} policy classes. For $f: \cS\times\cA\to\bbR$,
define
\[\pi_f(a\mid{}s)=\piref(a\mid{}s)e^{\frac{f(s,a)-V_f(s)}{\beta}},\quad
\text{where}\quad
V_f(s)=\beta\log\prn*{\sum_{a\in\cA}\piref(a\mid{}s)e^{\frac{f(s,a)}{\beta}}}.\]
We
consider policy classes of the form
\[
\Pi_\cF\ldef{}\crl*{\pi_f\mid{}f\in\cF}
\]
for a given value function class
$\cF\subseteq(\cS\times\cA\to\Rmax)$. Note that for such a class, we
can take $\Vmax\leq\Rmax$, and that $\Qstarb\in\cF$ implies
that $\pistarb\in\Pi_\cF$.

The following lemma bounds the SEC for log-linear policy classes in
terms of a preference-based analogue of the value function SEC in \citet{xie2023role}.

\begin{lemma}[SEC for log-linear policies]
  \label{lem:sec_value}
  For any value function class $\cF\subseteq(\cS\times\cA\to\Rmax)$,
  we have that $\coeff(\Pi,T,\beta;\samp)\leq\coeff(\cF,T;\samp)$, where
  \begin{small}
    \begin{align*}
      &\coeff(\cF,T;\samp)
        \ldef{} \sup_{f\ind{1},\ldots,f\ind{T}\in\cF}\\
      &\crl*{
        \sum_{t=1}^{T}\frac{\prn*{\E_{s_1 \sim \rho, \tau \sim
        \pi^\iter{t}\mid s_1, \taut \sim \pisamp\ind{t-1} \mid s_1} \left[\sum_{h=1}^{H}(f\ind{t}(s_h,a_h)-\brk*{\cT_{\beta}f\ind{t}}(s_h,a_h))
        - (f\ind{t}(\stil_h,\atil_h)-\brk*{\cT_{\beta}f\ind{t}}(\stil_h,\atil_h))
        \right]}^2}{\Rmax^2\vee
        (t-1)\cdot{}\E_{s_1 \sim \rho, (\tau,\taut) \sim \bmu\ind{t} \mid s_1} \left[\prn*{\sum_{h=1}^{H}(f\ind{t}(s_h,a_h)-\brk*{\cT_{\beta}f\ind{t}}(s_h,a_h))
        - (f\ind{t}(\stil_h,\atil_h)-\brk*{\cT_{\beta}f\ind{t}}(\stil_h,\atil_h))}^2 \right]} 
        },
    \end{align*}
    where $\pi\ind{t}\ldef{}\pi_{f\ind{t}}$,
    $\pisamp\ind{t}=\samp(\pi\ind{1},\ldots,\pi\ind{t})$, and 
    $\bmu\ind{t}\ldef{}\frac{1}{t-1}\sum_{i<t}\pi\ind{i}\otimes\pisamp\ind{i}$ (with the
      convention that $\mu\ind{1}$ is arbitrary), and where $\cT_\beta$ is the KL-regularized Bellman operator defined
    in \cref{sec:regularized}.
  \end{small}
\end{lemma}
\begin{proof}[\pfref{lem:sec_value}]
  This is an immediate corollary of \cref{lem:implicit_q_general}.
\end{proof}

We first apply this bound to give a polynomial bound on the SEC in tabular DCMDPs where $\cS$ and $\cA$ are finite.
\begin{example}[Tabular MDP]
  \label{ex:tabular}
  Suppose that $\samp$ sets $\pi\ind{t}=\pisamp$ for all $t$ for some
  fixed policy $\pitil$.
  When $\cF=\crl*{f:\cS\times\cA\to\Rmax}$ consists of all functions
  over tabular state and action spaces with
  $\abs{\cS},\abs{\cA}<\infty$, we have
  $\coeff(\cF,T;\samp)\leq\bigoht(H\abs{\cS}\abs{\cA})$ and
  $\log\abs{\Pi_\cF}\approxleq\bigoht(\abs{\cS}\abs{\cA})$. It follows that
  \algshort (\cref{alg:opt_dpo}) achieves
  \[
    \Jb(\pistarb) - \Jb(\pihat)
    \approxleq
    \bigoht\prn*{\Rmax{}e^{2\Rmax}\sqrt{\frac{H\abs{\cS}^2\abs{\cA}^2}{T}}}.
  \]
\end{example}
\cref{ex:tabular} is a corollary of the following more general result.
\begin{example}[Linear MDP]
  \label{ex:linear}
  In a Linear MDP \citep{jin2020provably}, we have
  \begin{align}
    P(s'\mid{}s,a)=\tri*{\phi(s,a),\mu(s')},
  \end{align}
  and
    \begin{align}
      r(s,a)=\tri*{\phi(s,a),\vartheta},
  \end{align}
  where $\phi(s,a)\in\bbR^{d}$ is a known feature map with
  $\nrm*{\phi(s,a)}\leq{}1$, $\mu(s')\in\bbR^{d}$ is an unknown
  feature map with $\nrm*{\sum_{s'}\mu(s')}\leq\sqrt{d}$, and
  $\varphi\in\bbR^{d}$ is an unknown parameter with $\nrm{\varphi}\leq{}1$. Here, the optimal
  KL-regularized value function $\Qstarb$ (cf. \cref{sec:regularized}) is
  linear with respect to the feature map $\phi(s,a)$. In particular,
  if we take
  \[
    \cF\ldef{}\crl*{f(s,a)=\tri*{\phi(s,a),\theta}\mid{}\theta\in\bbR^{d},\nrm{\theta}\leq{}B,\abs{f(s,a)}\leq{}R}
  \]
  for $B=\bigoh(\sqrt{d})$ and $R=\bigoh(\Rmax)$, then $\pistarb\in\Pi_\cF$, satisfying
  \cref{ass:realizability}. For this setting, when $\samp$ sets $\pi\ind{t}=\pisamp$ for all $t$ for some
  fixed policy $\pitil$, we have $\coeff(\cF,T;\samp)\leq\bigoht(d)$
  and $\log\abs{\Pi_\cF}\approxleq\bigoht(d)$. It follows that
  \algshort (\cref{alg:opt_dpo}) achieves
  \[
    \Jb(\pistarb) - \Jb(\pihat)
    \approxleq
    \bigoht\prn*{\Rmax{}e^{2\Rmax}\sqrt{\frac{d^2}{T}}}.
  \]
  \end{example}

\subsection{KL-Regularized MDP Preliminaries and
  $Q^{\star}$-Approximation}
\label{sec:regularized}

In this section, we give some basic background on value functions and dynamic
    programming for the KL-regularized MDP
    \citep{ziebart2008maximum,ziebart2010modeling}, then use these
    properties to prove \cref{lem:implicit_q,lem:implicit_q_general},
    which show that the optimal KL-regularized policy implicitly
    performs models rewards and performs $\Qstar$-approximation.

    \paragraph{Dynamic programming and value functions for
      KL-regularized MDP}
    First, for any
    function $f:\cS\times\cA\to\bbR$, define
    \[
      V_f(s_h) \coloneqq ~ \beta \log \sum_{a_h \in \Acal} \piref(a_h \mid s_h) e^{\nicefrac{f(s_h,a_h)}{\beta}}\quad\forall{}s\in\cS_h.
    \]
    It is straightforward to verify that
    \begin{align}
      \label{eq:vf_argmax}
V_f(s_h) = \max_{\pi:\cS\to\Delta(\cA)} \left( \E_{a_h \sim \pi(\cdot \mid s_h)} \left[ f(s_h,a_h) - \beta \log\frac{\pi(a_h \mid s_h)}{\piref(a_h \mid s_h)} \right] \right),
    \end{align}
    and that the policy that obtains the maximum above is
    \begin{align}
      \label{eq:vf_optpolicy} \pi_f(a_h\mid{}s_h) = \piref(a_h\mid{}s_h)e^{(f(s_h,a_h)-V_f(s_h))/\beta}. \
    \end{align}
    From here, beginning with $\Qstarb(s_H,a_H)\ldef{}r(s_H,a_H)$,
    $\pistarb(a_H\mid{}s_H)=\pi_{\Qstarb}(a_H\mid{}s_H)$, and $\Vstarb(s_H)=V_{\Qstarb}(s_H)$
    for $s_H\in\cS_H$, for each $s_h\in\cS_h$, we can inductively
    define for each $h\in\brk{H}$:
    \begin{align}
      \label{eq:soft_dp}
      \begin{aligned}
        &\Qstarb(s_h,a_h) \ldef r(s_h,a_h) +
        \En\brk*{\Vstarb(s_{h+1})\mid{}s_h,a_h},\\
        &\pistarb(a_h\mid{}s_h) \ldef \pi_{\Qstarb}(a_h\mid{}s_h),\\
        &\Vstarb(s_h) \ldef V_{\Qstarb}(s_h).
      \end{aligned}
    \end{align}
    In light of \cref{eq:vf_argmax}, it is clear that
    $\pistarb\in\argmax_{\pi:\cS\to\Delta(\cA)}\Jb(\pi)$. In addition,
    if we define the KL-regularized Bellman operator as
    \[
      \brk*{\Tcal_\beta f}(s_h,a_h) \coloneqq ~ r(s_h,a_h) + \E_{s_{h+1} \sim P(\cdot \mid s_h,a_h)} \left[ V_f(s_{h+1}) \right],
    \]
    we have that
    \[
      \Qstarb(s_h,a_h) = \brk*{\cT_\beta\Qstarb}(s_h,a_h).
    \]

\paragraph{Implicit $\Qstar$-approximation}
    
The following lemma, generalizing \citet{rafailov2024r}, shows that the
  optimal KL-regularized policy $\pistarb$ can be viewed as implicitly
  modeling rewards.
  \begin{restatable}[Implicit $Q^{\star}$-Approximation]{lemma}{implicitq}
    \label{lem:implicit_q}
    For any DCMDP, it holds that for all admissible\footnote{We use
      ``admissible" to a refer to a trajectory generated by executing
      an arbitrary policy $\pi:\cS\to\Delta(\cA)$ in the MDP.} trajectories $\tau=(s_1,a_1),\ldots,(s_H,a_H)$,
    \begin{equation}
      \label{eq:implicit_q}
      \beta\log\frac{\pistarb(\tau)}{\piref(\tau)} = r(\tau)
      -\Vstar_\beta(s_1),
    \end{equation}
    where $\Vstarb$ is the KL-regularized value function defined in \cref{eq:soft_dp}.
  \end{restatable}
    \begin{proof}[\pfref{lem:implicit_q}]
      Let
    $\tau=(s_1,a_1),\ldots,(s_H,a_H)$, and recall that for any DCMDP,
    all state transitions except for $s_1\sim\rho$ are
    deterministic. Then we have
\begin{align*}
0 = &~ \sum_{h=1}^H 
\left(Q^\star_\beta(s_h,a_h) - \brk*{\Tcal_\beta Q^\star_\beta}(s_h,a_h)\right)
\\
= &~ \sum_{h=1}^H 
\left(Q^\star_\beta(s_h,a_h) - r(s_h,a_h) - V^\star_\beta(s_{h+1})\right)
\\
= &~ \sum_{h=1}^H 
\left(V^\star_\beta(s_h) + \beta\log\frac{\pi^\star_\beta(a_h \mid s_h)}{\piref(a_h \mid s_h)} - r(s_h,a_h) - V^\star_\beta(s_{h+1})\right)
\\
= &~ V^\star_\beta(s_1) + \sum_{h=1}^H \left(\beta\log\frac{\pi^\star_\beta(a_h \mid s_h)}{\piref(a_h \mid s_h)} - r(s_h,a_h) \right),
\end{align*}
where the second equality uses that $(\Tcal_\beta f)(s_h,a_h)
  = r(s_h,a_h) + V_f(s_{h+1})$ for any admissible trajectory in a
  deterministic MDP, and the third equality uses the explicit form for $\pistarb$ in terms of $\Vstarb$ and $\Qstarb$ given in \Cref{eq:vf_optpolicy}. Rearranging
  yields the result.
  \end{proof}

We can also prove the following, more general version of \cref{lem:implicit_q_general}.
    \begin{lemma}[Implicit $Q^{\star}$-Approximation (general version)]
    \label{lem:implicit_q_general}
    For any DCMDP, it holds that for any function
    $f:\cS\times\cA\to\bbR$ and all admissible trajectories $\tau=(s_1,a_1),\ldots,(s_H,a_H)$,
    \begin{equation}
      \label{eq:implicit_q_general}
      \beta\log\frac{\pi_f(\tau)}{\piref(\tau)} = r(\tau)
      -V_f(s_1) + \sum_{h=1}^H 
\left(f(s_h,a_h) - \brk*{\Tcal_\beta f}(s_h,a_h)\right).
    \end{equation}
  \end{lemma}
    \begin{proof}[\pfref{lem:implicit_q_general}]
Let
    $\tau=(s_1,a_1),\ldots,(s_H,a_H)$. Then we have
\begin{align*}
&~ \sum_{h=1}^H 
\left(f(s_h,a_h) - \brk*{\Tcal_\beta f}(s_h,a_h)\right)
\\
= &~ \sum_{h=1}^H 
\left(f(s_h,a_h) - r(s_h,a_h) - V_f(s_{h+1})\right)
\\
= &~ \sum_{h=1}^H 
\left(V_f(s_h) + \beta\log\frac{\pi_f(a_h \mid s_h)}{\piref(a_h \mid s_h)} - r(s_h,a_h) - V_f(s_{h+1})\right)
\\
= &~ V_f(s_1) + \sum_{h=1}^H \left(\beta\log\frac{\pi_f(a_h \mid s_h)}{\piref(a_h \mid s_h)} - r(s_h,a_h) \right),
\end{align*}
where the first equality uses the definition of $V_f$, the second equality uses that $(\Tcal_\beta f)(s_h,a_h)
  = r(s_h,a_h) + V_f(s_{h+1})$ for any admissible trajectory in a
  deterministic MDP, and the third equality uses that
  $\pi_f(a\mid{}s)=\piref(a\mid{}s)e^{\frac{f(s,a)-V_f(s)}{\beta}}$. Rearranging
  yields the result.
  \end{proof}

  \subsection{Regret Decomposition}
  \label{sec:regret_decomp}

In this section we prove the central regret decomposition for
\algshort, restated below.
  
  \central*
    \begin{proof}[\pfref{lem:regret_decomp}]
    It follows immediately from the definition of the KL-regularized
    reward that
    \begin{align*}
      J_\beta(\pi^\star_\beta) - J_\beta(\pi)
      =
      \En_{\pi}\brk*{\beta\log\frac{\pi(\tau)}{\piref(\tau)}-r(\tau)}
      - \En_{\pistarb}\brk*{\beta\log\frac{\pistarb(\tau)}{\piref(\tau)}-r(\tau)}.
    \end{align*}
    However, since
    $\beta\log\frac{\pistarb(\tau)}{\piref(\tau)}-r(\tau)=\Vstar_\beta(s_1)$
    for all admissible trajectories by \cref{lem:implicit_q}, we have
    that
    \begin{align*}
      \En_{\pistarb}\brk*{\beta\log\frac{\pistarb(\tau)}{\piref(\tau)}-r(\tau)}
      =   \En_{\nu}\brk*{\beta\log\frac{\pistarb(\tau)}{\piref(\tau)}-r(\tau)}
    \end{align*}
    for all policies $\nu$, as the initial state $s_1$ does not
    depend on the policy under consideration. The result now follows
    by rearranging
    \begin{align*}
      &\En_{\pi}\brk*{\beta\log\frac{\pi(\tau)}{\piref(\tau)}-r(\tau)}
      -
      \En_{\nu}\brk*{\beta\log\frac{\pistarb(\tau)}{\piref(\tau)}-r(\tau)}\\
      &=
      \E_{\nu} \left[\beta\log\pi(\tau)\right] - \E_{\nu} \left[\beta\log\pi^\star_\beta(\tau)\right] 
+ \E_{\pi} \left[
\beta\log\frac{\pi(\tau)}{\piref(\tau)} - r(\tau) \right] - \E_{\nu} \left[\beta\log\frac{\pi(\tau)}{\piref(\tau)} - r(\tau)\right].
    \end{align*}  
  \end{proof}

\subsection{Concentration Lemmas}
\label{sec:conc}

Recall that we define
$\bmu\ind{t}=\frac{1}{t-1}\sum_{i<t}\pi\ind{i}\otimes\pisamp\ind{i}$. For
a given policy $\pi$, define
      \begin{align*}
        f_{\pi}(\tau,\tautil)
        =       \beta\log\frac{\pi(\tau)}{\piref(\tau)}
        - \beta\log\frac{\pi(\tautil)}{\piref(\tautil)}.
      \end{align*}
The following lemma is our central concentration guarantee for \cref{alg:opt_dpo}.
\begin{restatable}[Concentration for \algshort]{lemma}{concmain}
  \label{lem:conc_main}
Suppose that \cref{ass:vmax,ass:realizability} hold. Then
  \cref{alg:opt_dpo} guarantees that with probability at least
  $1-\delta$, for all steps $t\in\brk{T}$,
  \begin{align*}
          &\alpha\cdot\En_{s_1\sim\rho,\tau\sim\pitil\ind{t-1}}\brk*{\log(\pi\ind{t}(\tau))-\log(\pistarb(\tau))}
+\kappa\cdot{}\En_{s_1\sim\rho{},(\tau,\tautil)\sim{}\bmu\ind{t}\mid{}s_1}\brk*{\prn*{f_{\pi\ind{t}}(\tau,\tautil)
        - f_{\pistarb}(\tau,\tautil)}^2}\notag\\
&\leq \frac{2\log(2\abs{\Pi}T\delta^{-1})}{t-1} + \frac{\alpha}{\beta}\Vmax\sqrt{\frac{2^4\log(2\abs{\Pi}T\delta^{-1})}{t-1}},           
  \end{align*}
  for $\kappa\ldef{}(8(\Rmax+\Vmax)e^{2\Rmax})^{-2}$.
\end{restatable}

  \begin{proof}[\pfref{lem:conc_main}]%
    Let $t\in\crl*{2,\ldots,T+1}$ be fixed.
    \begin{align}
      &\Lhat\ind{t}(\pi)\notag\\&=\sum_{i<t}-y\ind{i}\log\left[\sigma\left(
      \beta\log\frac{\pi(\tau\ind{i})}{\piref(\tau\ind{i})} -
      \beta\log\frac{\pi(\wt{\tau}\ind{i})}{\piref(\wt{\tau}\ind{i})} \right) \right]
      -(1-y\ind{i}) \log\left[\sigma\left(
      \beta\log\frac{\pi(\tautil\ind{i})}{\piref(\tautil\ind{i})} -
      \beta\log\frac{\pi(\tau\ind{i})}{\piref(\tau\ind{i})} \right) \right]\label{eq:lhat}
    \end{align}
    and
    $\Bhat\ind{t}(\pi)=\alpha\sum_{\tau\in\cDopt\ind{t-1}}\log\pi(\tau)$. Then
    we can equivalently write
    \begin{align*}
      \pi\ind{t}=\argmin_{\pi\in\Pi}\crl*{\Lhat\ind{t}(\pi) + \Bhat\ind{t}(\pi)}.
    \end{align*}
          For a given policy $\pi$, recall that we define
      \begin{align*}
        f_{\pi}(\tau,\tautil)
        =       \beta\log\frac{\pi(\tau)}{\piref(\tau)}
        - \beta\log\frac{\pi(\tautil)}{\piref(\tautil)},
      \end{align*}
      and let
      \begin{align*}
        P_{\pi}(y\mid{}\tau,\tautil) = y\cdot{}\sigma(f_{\pi}(\tau,\tautil))
        + (1-y)\cdot{}(1-\sigma(f_{\pi}(\tau,\tautil))).
      \end{align*}
      Then, in light of \cref{lem:implicit_q}, under the
      Bradley-Terry model (\cref{eq:bt}), we have that for all $t$,
      \begin{align}
        \label{eq:yt_law}
        y\ind{t}\sim{}P_{\pistarb}(\cdot\mid{}\tau\ind{t},\tautil\ind{t}).
      \end{align}
      In addition, we can rewrite \cref{eq:lhat} as
      \begin{align*}
        \Lhat(\pi) =
        \sum_{i<t}-\log(P_{\pi}(y\ind{t}\mid{}\tau\ind{t},\tautil\ind{t})).
      \end{align*}
Using this observation, we begin by proving an intermediate
concentration result. For a pair of probability measures $\bbP$ and $\bbQ$, we define
squared Hellinger distance via
\begin{equation}
  \label{eq:hellinger}
  \Dhels{\bbP}{\bbQ}=\int\prn*{\sqrt{d\bbP}-\sqrt{d\bbQ}}^2.
\end{equation}
    \begin{lemma}
      \label{lem:conc_mle}
      For any fixed $t\geq{}1$, with probability at least $1-\delta$, all $\pi\in\Pi$ satisfy
            \begin{align*}
        \sum_{i<t}
\En_{s_1\sim\rho{},\tau\sim{}\pi\ind{i}\mid{}s_1,\tautil\sim\pisamp\ind{i}\mid{}s_1}\brk*{\Dhels{P_{\pi}(\cdot\mid{}\tau,\tautil)}{P_{\pistarb}(\cdot\mid{}\tau,\tautil)
                  }}
\leq        \Lhat\ind{t}(\pi)-\Lhat\ind{t}(\pistarb) + 2\log(\abs{\Pi}\delta^{-1}).
      \end{align*}
    \end{lemma}    
Rearranging \cref{lem:conc_mle}, with probability at least $1-\delta$,
all $\pi\in\Pi$ satisfy
\begin{align*}
  &\Bhat\ind{t}(\pi) - \Bhat\ind{t}(\pistarb) + \sum_{i<t}
\En_{s_1\sim\rho{},\tau\sim{}\pi\ind{i}\mid{}s_1,\tautil\sim\pisamp\ind{i}\mid{}s_1}\brk*{\Dhels{P_{\pi}(\cdot\mid{}\tau,\tautil)}{P_{\pistarb}(\cdot\mid{}\tau,\tautil)
                  }}\\
&\leq        \Lhat\ind{t}(\pi)+\Bhat\ind{t}(\pi)-\Lhat\ind{t}(\pistarb)-\Bhat\ind{t}(\pistarb) + 2\log(\abs{\Pi}\delta^{-1}).
\end{align*}
Hence, as long as $\pistarb\in\Pi$ (\cref{ass:realizability}), the definition of $\pi\ind{t}$ in
\cref{alg:general} implies that
\begin{align}
  \label{eq:pit_central}
  &\Bhat\ind{t}(\pi\ind{t}) - \Bhat\ind{t}(\pistarb) + \sum_{i<t}
\En_{s_1\sim\rho{},\tau\sim{}\pi\ind{i}\mid{}s_1,\tautil\sim\pisamp\ind{i}\mid{}s_1}\brk*{\Dhels{P_{\pi\ind{t}}(\cdot\mid{}\tau,\tautil)}{P_{\pistarb}(\cdot\mid{}\tau,\tautil)
                  }}
\leq 2\log(\abs{\Pi}\delta^{-1}).
\end{align}
We next appeal to another basic concentration result.
\begin{lemma}
  \label{lem:conc_bonus}
      For any fixed $t\geq{}1$, with probability at least $1-\delta$,
      all $\pi\in\Pi$ satisfy
      \begin{align*}
        \alpha\cdot(t-1)\cdot\En_{s_1\sim\rho,\tau\sim\pisamp\ind{t-1}\mid{}s_1}\brk*{\log(\pi(\tau))-\log(\pistarb(\tau))}
        \leq{} \Bhat\ind{t}(\pi) - \Bhat\ind{t}(\pistarb)
        + \frac{\alpha}{\beta}\Vmax\sqrt{2^4(t-1)\log(\abs{\Pi}\delta^{-1})}.
      \end{align*}
    \end{lemma}
    Combining \cref{lem:conc_bonus} with \cref{eq:pit_central}, we conclude that with
    probability at least $1-2\delta$,
    \begin{align*}
      &\alpha\cdot(t-1)\cdot\En_{s_1\sim\rho,\tau\sim\pisamp\ind{t-1}\mid{}s_1}\brk*{\log(\pi\ind{t}(\tau))-\log(\pistarb(\tau))}
+\sum_{i<t}\En_{s_1\sim\rho{},\tau\sim{}\pi\ind{i}\mid{}s_1,\tautil\sim\pisamp\ind{i}\mid{}s_1}\brk*{\Dhels{P_{\pi\ind{t}}(\cdot\mid{}\tau,\tautil)}{P_{\pistarb}(\cdot\mid{}\tau,\tautil)
                  }}\\
&\leq 2\log(\abs{\Pi}\delta^{-1}) + \frac{\alpha}{\beta}\Vmax\sqrt{2^6(t-1)\log(\abs{\Pi}\delta^{-1})},
    \end{align*}
    or equivalently,
    \begin{align}
      &\alpha\cdot\En_{s_1\sim\rho,\tau\sim\pisamp\ind{t-1}\mid{}s_1}\brk*{\log(\pi\ind{t}(\tau))-\log(\pistarb(\tau))}
+\En_{s_1\sim\rho{},(\tau,\tautil)\sim{}\bmu\ind{t}\mid{}s_1}\brk*{\Dhels{P_{\pi\ind{t}}(\cdot\mid{}\tau,\tautil)}{P_{\pistarb}(\cdot\mid{}\tau,\tautil)
                  }}\notag\\
&\leq \frac{2\log(\abs{\Pi}\delta^{-1})}{t-1} + \frac{\alpha}{\beta}\Vmax\sqrt{\frac{2^6\log(\abs{\Pi}\delta^{-1})}{t-1}},            \label{eq:pit2_central2}
    \end{align}

    To conclude, we further simplify the expression via
    \begin{align*}
&\En_{s_1\sim\rho{},(\tau,\tautil)\sim{}\bmu\ind{t}\mid{}s_1}\brk*{\Dhels{P_{\pi\ind{t}}(\cdot\mid{}\tau,\tautil)}{P_{\pistarb}(\cdot\mid{}\tau,\tautil)}}\\
      &      \geq
      \En_{s_1\sim\rho{},(\tau,\tautil)\sim{}\bmu\ind{t}\mid{}s_1}\brk*{\prn*{\sqrt{\sigma(f_{\pi\ind{t}}(\tau,\tautil))}
        - \sqrt{\sigma(f_{\pistarb}(\tau,\tautil))}}^2}\\
            &      \geq
              \frac{1}{8}\En_{s_1\sim\rho{},(\tau,\tautil)\sim{}\bmu\ind{t}\mid{}s_1}\brk*{\prn*{\sigma(f_{\pi\ind{t}}(\tau,\tautil))
        - \sigma(f_{\pistarb}(\tau,\tautil))}^2},
    \end{align*}
    where the last inequality uses that for $x,y\geq{}0$,
    $(x-y)^2\leq{}4(x+y)(\sqrt{x}-\sqrt{y})^2$.

    Finally, using \cref{lem:implicit_q}, we have
    $f_{\pistarb}\in\brk*{-\Rmax,\Rmax}$ almost surely, while
    $f_{\pi\ind{t}}\in\brk*{-\Vmax,\Vmax}$ by
    \cref{ass:vmax}. We appeal to the following lemma.
    \begin{lemma}[e.g., \citet{rosset2024direct}]
      \label{lem:sigmoid}
      If $x\in\brk*{-X,X}$ and $y\in\brk{-Y,Y}$ for $X\geq{}0$, $Y\geq{}1$, then
      \begin{align*}
        \abs*{x-y}\leq{}
         8\prn{X+Y}e^{2Y}\abs*{\sigma(x)-\sigma(y)}.
      \end{align*}
    \end{lemma}
        From this, we conclude that
        \begin{align*}
          &\En_{s_1\sim\rho{},(\tau,\tautil)\sim{}\bmu\ind{t}\mid{}s_1}\brk*{\prn*{\sigma(f_{\pi\ind{t}}(\tau,\tautil))
          - \sigma(f_{\pistarb}(\tau,\tautil))}^2}\\
          &\geq{} (8(\Rmax+\Vmax)e^{2\Rmax})^{-2}\cdot{}\En_{s_1\sim\rho{},(\tau,\tautil)\sim{}\bmu\ind{t}\mid{}s_1}\brk*{\prn*{f_{\pi\ind{t}}(\tau,\tautil)
        - f_{\pistarb}(\tau,\tautil)}^2}
        \end{align*}
This proves the result after taking a union bound over all steps $t$.

  \end{proof}

\subsubsection{Proofs for Supporting Lemmas}
  
      \begin{proof}[\pfref{lem:conc_mle}]
      To begin, define
      \[
        \ell\ind{i}(\pi) =-\log(P_{\pi}(y\ind{t}\mid{}\tau\ind{t},\tautil\ind{t})).
      \]
      For a fixed policy $\pi\in\Pi$, define
      $Z\ind{i}(\pi)=\frac{1}{2}(\ell\ind{i}(\pi)-\ell\ind{i}(\pistarb))$. Define
      a filtration
      $\filt\ind{t}=\sigma((\tau\ind{1},\tautil\ind{1}),\ldots,(\tau\ind{t-1},\tautil\ind{t-1}))$. Applying
      \cref{lem:martingale_chernoff} with the sequence $(Z_i(\pi))$
      and taking a union bound over $\pi\in\Pi$, have that with
      probability at least $1-\delta$, all $\pi\in\Pi$ satisfy
      \begin{align*}
        -\sum_{i<t}\log\prn*{\En_{i-1}\brk*{\exp\prn*{-\frac{1}{2}Z_i(\pi)}}
        }
        \leq{} \frac{1}{2}\prn*{\Lhat\ind{t}(\pi)-\Lhat\ind{t}(\pistarb)} + \log(\abs{\Pi}\delta^{-1}).
      \end{align*}
      Next, using \cref{eq:yt_law} and a somewhat standard argument
      from \citet{Sara00,zhang2006from}, we calculate that
      \begin{align*}
        &\En_{i-1}\brk*{\exp\prn*{\frac{1}{2}Z_i(\pi)}}\\
        &=
          \En_{s_1\sim\rho{},\tau\sim{}\pi\ind{i}\mid{}s_1,\tautil\sim\pisamp\ind{i}\mid{}s_1,y\sim{}P_{\pistarb}(\cdot\mid{}\tau,\tautil)}
          \brk*{\exp\prn*{\frac{1}{2}\log(P_{\pi}(y\mid{}\tau,\tautil)/P_{\pistarb}(y\mid{}\tau,\tautil))
          }
          }\\
                &=
                  \En_{s_1\sim\rho{}, \tau\sim{}\pi\ind{i}\mid{}s_1,\tautil\sim\pisamp\ind{i}\mid{}s_1}\brk*{\sum_{y\in\crl*{0,1}}
                  \sqrt{P_{\pi}(y\mid{}\tau,\tautil) P_{\pistarb}(y\mid{}\tau,\tautil)}
                  }\\
                        &=
\En_{s_1\sim\rho{},\tau\sim\pi\ind{i}\mid{}s_1,\tautil\sim\pisamp\ind{i}\mid{}s_1}\brk*{1-\frac{1}{2}\Dhels{P_{\pi}(\cdot\mid{}\tau,\tautil)}{P_{\pistarb}(\cdot\mid{}\tau,\tautil)
                  }}.
      \end{align*}
      Since $\Dhels{\cdot}{\cdot}\leq{}2$ and $-\log(1-x)\geq{}x$ for
      $x\leq{}1$, we conclude that
      \begin{align*}
        \sum_{i<t}
\En_{s_1\sim\rho{},\tau\sim{}\pi\ind{i}\mid{}s_1,\tautil\sim\pisamp\ind{i}\mid{}s_1}\brk*{\Dhels{P_{\pi}(\cdot\mid{}\tau,\tautil)}{P_{\pistarb}(\cdot\mid{}\tau,\tautil)
                  }}
\leq        \Lhat\ind{t}(\pi)-\Lhat\ind{t}(\pistarb) + 2\log(\abs{\Pi}\delta^{-1})
      \end{align*}
      
    \end{proof}

    \begin{proof}[\pfref{lem:conc_bonus}]
      Let $\tau\ind{1},\ldots,\tau\ind{t-1}$ denote the trajectories
      in $\cDopt\ind{t-1}$. Let $\bhat\ind{i}(\pi) =
\alpha\log\pi(\tau\ind{i})$, and
let
\[
Z\ind{i}(\pi) = \bhat\ind{i}(\pi)-\bhat\ind{i}(\pistarb).
\]
We can equivalently re-write this as
\[
  Z\ind{i}(\pi) = \alpha\prn*{
    \log\prn*{\frac{\pi(\tau\ind{i})}{\piref(\tau\ind{i})}}
    -\log\prn*{\frac{\pistarb(\tau\ind{i})}{\piref(\tau\ind{i})}}}
    ,
  \]
  which implies that $\abs*{Z\ind{i}(\pi)}\leq{}2\frac{\alpha}{\beta}\Vmax$. From
  here, the result follows immediately by applying
  \cref{lem:hoeffding} with the sequence $(Z_i(\pi))$ and taking a
  union bound over $\pi\in\Pi$.
\end{proof}

    \begin{proof}[\pfref{lem:sigmoid}]
      We consider three cases. First, if $x\in\brk*{-2Y,2Y}$, then
      \begin{align*}
        \abs*{\sigma(x)-\sigma(y)}
        \geq{} \sigma'(z)\abs*{x-y}
      \end{align*}
      for some $z\in\brk*{-2Y,2Y}$. In this regime, we have
      $\sigma'(z)\geq{}\sigma'(2Y) =
      e^{2Y}/(1+e^{2Y})^2\geq{}(4e^{2Y})^{-1}$. Next, if
      $x\geq{}2Y>0$, we can directly bound
      \begin{align*}
        \sigma(x)-\sigma(y)\geq{}\sigma(2Y)-\sigma(Y) =
        \frac{e^{2Y}-e^{Y}}{(1+e^{2Y})(1+e^Y)}
        \geq{} \frac{1-e^{-Y}}{4e^{Y}}\geq{}\frac{1}{8e^{Y}},
      \end{align*}
      where the last line holds whenever $Y\geq{}1$. We conclude in
      this case that
      \begin{align*}
        \frac{\abs*{x-y}}{\sigma(x)-\sigma(y)}\leq{}
        \frac{X+Y}{\sigma(x)-\sigma(y)}\leq{} 8\prn{X+Y}e^{Y}.
      \end{align*}
      Finally, we consider the case where $x\leq{}-2Y\leq{}0$. In this
      case, we can similarly lower bound
      \begin{align*}
        \sigma(y)-\sigma(x)\geq{}\sigma(-Y)-\sigma(-2Y)
        =\frac{e^{-Y}-e^{-2Y}}{(1+e^{-Y})(1+e^{-2Y})}
        \geq{} \frac{1-e^{-Y}}{4e^{2Y}}\geq{}\frac{1}{8e^{2Y}}
      \end{align*}
      as long as $Y\geq{}1$. From here, proceeding in the same fashion as the
      second case yields the result.
    \end{proof}

    \subsection{Proof of \creftitle{thm:main_general}}
    \label{sec:main_proof}

\begin{proof}[\pfref{thm:main_general}]
  Before diving into the proof, we re-state two central technical lemmas. The
  first lemma, generalizing \citet{rafailov2024r}, shows that the
  optimal KL-regularized policy $\pistarb$ can be viewed as implicitly
  modeling rewards.

  \implicitq*
  
  This lemma allows us to view the \dpo objective as a form of
  implicit $Q^{\star}$-approximation. Building on this lemma, we prove
  the following regret decomposition.
  
  \central*
  
      This result shows that the (regularized) regret of any policy $\pi$ can
    be decomposed into two terms. The term in \cref{eq:decomp2}
    measures the extent to which $\pi$ (implicitly) models the reward;
    by \cref{lem:implicit_q}, this term is zero when
    $\pi=\pistarb$. Meanwhile, the term in \cref{eq:decomp1} measures
    the extent to which the policy $\pi$ over-estimates the internal
    reward; we will control this term using optimism. Importantly, the
    regret decomposition in \cref{lem:regret_decomp} holds for an
    arbitrary roll-in policy $\nu$. This will facilitate minimizing
    the terms in the regret decomposition in a data-driven
    fashion. Before proceeding, we remark that \cref{lem:implicit_q}
    and \cref{lem:regret_decomp} together imply that
    \begin{align}
      \label{eq:regret_range}
      \Jb(\pistarb) - \Jb(\pi) \leq{} 6\Vmax
    \end{align}
    for all $\pi\in\Pi$.

    We now begin the proof by writing
    \begin{align*}
      \Jb(\pistarb) - \Jb(\pihat)
      = \min_{t\in\brk{T+1}}\Jb(\pistarb) - \Jb(\pi\ind{t})
      \leq{} \frac{1}{T}\sum_{t=1}^{T}\Jb(\pistarb) - \Jb(\pi\ind{t}).
    \end{align*}
    For each step $t$, we apply \cref{lem:regret_decomp} with
    $\pi=\pi\ind{t}$ and $\nu=\pisamp\ind{t-1}$, which gives
    \begin{align}
      &\frac{1}{T}\sum_{t=1}^{T}\Jb(\pistarb) - \Jb(\pi\ind{t})\notag\\
      &\leq{} \frac{1}{T}\sum_{t=1}^{T}\E_{\tau \sim \pisamp\ind{t-1}}
        \left[\beta\log\pi\ind{t}(\tau)-\beta\log\pi^\star_\beta(\tau)\right]\notag\\
      &~~~~+ \frac{1}{T}\sum_{t=1}^{T}\E_{\tau \sim \pi\ind{t}} \left[
\beta\log\frac{\pi\ind{t}(\tau)}{\piref(\tau)} - r(\tau) \right] -
        \E_{\tau\sim\pisamp\ind{t-1}}
        \left[\beta\log\frac{\pi\ind{t}(\tau)}{\piref(\tau)} -
        r(\tau)\right].\notag\\
            &= \frac{1}{T}\sum_{t=1}^{T}\E_{\tau \sim \pisamp\ind{t-1}}
              \left[\beta\log\pi\ind{t}(\tau)-\beta\log\pi^\star_\beta(\tau)\right]\notag\\
      &~~~~+ \frac{1}{T}\sum_{t=1}^{T}\E_{s_1\sim\rho,\tau \sim \pi\ind{t}\mid{}s_1,\taut\sim\pisamp\ind{t-1}\mid{}s_1} \left[
\beta\log\frac{\pi\ind{t}(\tau)}{\piref(\tau)} -
        r(\tau)-\beta\log\frac{\pi\ind{t}(\taut)}{\piref(\taut)} +
        r(\taut)\right].\notag\\
      &\leq{} \frac{6\Vmax}{T} + \frac{1}{T}\sum_{t=2}^{T}\E_{\tau \sim \pisamp\ind{t-1}}
              \left[\beta\log\pi\ind{t}(\tau)-\beta\log\pi^\star_\beta(\tau)\right]\label{eq:main0}\\
      &~~~~+ \frac{1}{T}\sum_{t=2}^{T}\E_{s_1\sim\rho,\tau \sim \pi\ind{t}\mid{}s_1,\taut\sim\pisamp\ind{t-1}\mid{}s_1} \left[
\beta\log\frac{\pi\ind{t}(\tau)}{\piref(\tau)} - r(\tau)-\beta\log\frac{\pi\ind{t}(\taut)}{\piref(\taut)} + r(\taut)\right],\notag
    \end{align}
    where the last line follows by \cref{eq:regret_range}.

    Next, recall that we define $\bmu\ind{t}=\frac{1}{t-1}\sum_{i<t}\pi\ind{t}\otimes\pisamp\ind{t}$
    Consider a fixed step $t\geq{}2$, and define
    \begin{align*}
      \cI\ind{t}
      \ldef{} \frac{\prn*{\E_{s_1 \sim \rho, \tau \sim
      \pi^\iter{t}\mid s_1, \taut \sim \pisamp\ind{t-1} \mid s_1} \left[
      \beta\log\frac{\pi^\iter{t}(\tau)}{\piref(\tau)} - r(\tau) -
      \beta\log\frac{\pi^\iter{t}(\taut)}{\piref(\taut)} + r(\taut)
      \right]}^2}{\Vmax^2\vee
      (t-1)\cdot{}\E_{s_1 \sim \rho, (\tau,\taut) \sim \bmu\ind{t} \mid s_1} \left[\prn*{ \beta\log\frac{\pi^\iter{t}(\tau)}{\piref(\tau)} - r(\tau) - \beta\log\frac{\pi^\iter{t}(\taut)}{\piref(\taut)} + r(\taut) }^2 \right]}.
    \end{align*}
    Then, using the AM-GM inequality, for any $\eta>0$ we can bound
    \begin{align}
&\E_{s_1\sim\rho,\tau \sim \pi\ind{t}\mid{}s_1,\taut\sim\pisamp\ind{t-1}\mid{}s_1} \left[
\beta\log\frac{\pi\ind{t}(\tau)}{\piref(\tau)} -
                     r(\tau)-\beta\log\frac{\pi\ind{t}(\taut)}{\piref(\taut)}
                     + r(\taut)\right]\notag\\
      &\leq{} \frac{\cI\ind{t}}{2\eta}
        + \frac{\eta}{2}\cdot{}\prn*{\Vmax^2\vee
      (t-1)\cdot{}\E_{s_1 \sim \rho, (\tau,\taut) \sim \bmu\ind{t} \mid
        s_1} \left[\prn*{
        \beta\log\frac{\pi^\iter{t}(\tau)}{\piref(\tau)} - r(\tau) -
        \beta\log\frac{\pi^\iter{t}(\taut)}{\piref(\taut)} + r(\taut)
        }^2 \right]}\notag\\
      &\leq{} \frac{\cI\ind{t}}{2\eta}
        + \frac{\eta}{2}\cdot{}\prn*{\Vmax^2+
        (t-1)\cdot{}\E_{s_1 \sim \rho, (\tau,\taut) \sim \bmu\ind{t} \mid s_1} \left[\prn*{ \beta\log\frac{\pi^\iter{t}(\tau)}{\piref(\tau)} - r(\tau) - \beta\log\frac{\pi^\iter{t}(\taut)}{\piref(\taut)} + r(\taut) }^2 \right]}.\label{eq:main1}
    \end{align}
    Note that by definition, we have that
    $\sum_{t=1}^{T}\cI\ind{t}\leq\coefffull$. Hence, by
    plugging \cref{eq:main1} into \cref{eq:main0} and summing, we
    conclude that
    \begin{align}
      &\frac{1}{T}\sum_{t=1}^{T}\Jb(\pistarb) - \Jb(\pi\ind{t})\notag\\
      &\leq{} \frac{6\Vmax}{T} + \frac{\coefffull}{2\eta{}T}
      + \frac{\eta}{2}\Vmax^2 + 
      \frac{1}{T}\sum_{t=2}^{T}\E_{\tau \sim \pisamp\ind{t-1}}
        \left[\beta\log\pi\ind{t}(\tau)-\beta\log\pi^\star_\beta(\tau)\right]\notag\\
      &~~~~+\frac{\eta}{2T}\sum_{t=2}^{T}(t-1)\cdot{}\E_{s_1 \sim \rho, (\tau,\taut) \sim \bmu\ind{t} \mid s_1} \left[\prn*{ \beta\log\frac{\pi^\iter{t}(\tau)}{\piref(\tau)} - r(\tau) - \beta\log\frac{\pi^\iter{t}(\taut)}{\piref(\taut)} + r(\taut) }^2 \right]\label{eq:main3}.
    \end{align}
    Fix
    $t$, and consider the term
    \begin{align}
      \label{eq:term0}
      \E_{\tau \sim \pisamp\ind{t-1}}
      \left[\beta\log\pi\ind{t}(\tau)-\beta\log\pi^\star_\beta(\tau)\right]
      + \frac{\eta(t-1)}{2}\E_{s_1 \sim \rho, (\tau,\taut) \sim \bmu\ind{t} \mid s_1} \left[\prn*{ \beta\log\frac{\pi^\iter{t}(\tau)}{\piref(\tau)} - r(\tau) - \beta\log\frac{\pi^\iter{t}(\taut)}{\piref(\taut)} + r(\taut) }^2\right]
    \end{align}
    above. 
        Let $f_{\pi}(\tau,\taut)\ldef{}
    \beta\log\frac{\pi(\tau)}{\piref(\tau)}-\beta\log\frac{\pi(\taut)}{\piref(\taut)}$. 
    By
    \cref{lem:implicit_q}, we have that for any pair of admissible trajectories
    $(\tau,\tautil)$ that share the initial state $s_1$,
    $f_{\pistarb}(\tau,\tautil)=r(\tau)-r(\tautil)$, so we can rewrite
    \cref{eq:term0} as
        \begin{align}
      \label{eq:term1}
      \E_{\tau \sim \pisamp\ind{t-1}}
      \left[\beta\log\pi\ind{t}(\tau)-\beta\log\pi^\star_\beta(\tau)\right]
          + \frac{\eta(t-1)}{2}\E_{s_1 \sim \rho, (\tau,\taut) \sim \bmu\ind{t} \mid s_1} \left[\prn*{f_{\pi\ind{t}}(\tau,\tautil)-f_{\pistarb}(\tau,\tautil)}^2\right].
        \end{align}
        We now recall the central concentration lemma for \algshort
        (\cref{lem:conc_main}).

        \concmain*
        
It follows that if we set $\eta = \frac{\beta\kappa}{\alpha{}T}\leq{}\frac{\beta\kappa}{\alpha(t-1)}$,
then with probability at least $1-\delta$, for all $t\in\brk{T}$,
\begin{align*}
  \text{\cref{eq:term1}}
  \approxleq
  \frac{\beta}{\alpha}\cdot{}\prn*{\frac{\log(\abs{\Pi}T\delta^{-1})}{t-1}
  + \frac{\alpha}{\beta}\Vmax\sqrt{\frac{\log(\abs{\Pi}T\delta^{-1})}{t-1}}}\\
  = \frac{\beta\log(\abs{\Pi}T\delta^{-1})}{\alpha(t-1)} + \Vmax\sqrt{\frac{\log(\abs{\Pi}T\delta^{-1})}{t-1}}.
\end{align*}
Plugging this bound back into \cref{eq:main3}, we have that
\begin{align*}
  &\frac{1}{T}\sum_{t=1}^{T}\Jb(\pistarb) - \Jb(\pi\ind{t})  \\
  &\approxleq{} \frac{\Vmax}{T} + \frac{\coefffull}{\eta{}T}
    + \eta\Vmax^2
    + \frac{1}{T}\sum_{t=2}^{T}
    \prn*{\frac{\beta\log(\abs{\Pi}T\delta^{-1})}{\alpha(t-1)}
    +
    \Vmax\sqrt{\frac{\log(\abs{\Pi}T\delta^{-1})}{t-1}}}\\
    &\approxleq{} \frac{\Vmax}{T} + \frac{\coefffull}{\eta{}T}
    + \eta\Vmax^2
    + \frac{\beta\log(\abs{\Pi}T\delta^{-1})\log(T)}{\alpha{}T}
    +
      \Vmax\sqrt{\frac{\log(\abs{\Pi}T\delta^{-1})}{T}}\\
      &= \frac{\Vmax}{T} + \frac{\alpha\cdot\coefffull}{\beta\kappa}
    + \frac{\beta\kappa\Vmax^2}{\alpha{}T}
    + \frac{\beta\log(\abs{\Pi}T\delta^{-1})\log(T)}{\alpha{}T}
    +
        \Vmax\sqrt{\frac{\log(\abs{\Pi}T\delta^{-1})}{T}}\\
  &\approxleq{} \frac{\alpha\cdot\coefffull}{\beta\kappa}
    + \frac{\beta\kappa\Vmax^2}{\alpha{}T}
    + \frac{\beta\log(\abs{\Pi}T\delta^{-1})\log(T)}{\alpha{}T}
    +
    \Vmax\sqrt{\frac{\log(\abs{\Pi}T\delta^{-1})}{T}}\\
    &\approxleq{} \frac{\alpha\cdot\coefffull}{\beta\kappa}
    + \frac{\beta\log(\abs{\Pi}T\delta^{-1})\log(T)}{\alpha{}T}
    +
    \Vmax\sqrt{\frac{\log(\abs{\Pi}T\delta^{-1})}{T}},
\end{align*}
where the last line uses that $\kappa\leq\Vmax^{-2}$. It follows that by choosing
\begin{align}
  \label{eq:alpha_choice}
\alpha \propto\sqrt{\frac{\beta\kappa\cdot{} \beta\log(\abs{\Pi}T\delta^{-1})\log(T)}{T\cdot{}\coefffull}},
\end{align}
we obtain
\begin{align}
&\frac{1}{T}\sum_{t=1}^{T}\Jb(\pistarb) - \Jb(\pi\ind{t})  \\
    &\approxleq
    \sqrt{\frac{
    \kappa^{-1}\log(\abs{\Pi}T\delta^{-1})\log(T))\cdot\coefffull}{T}}
    + 
      \Vmax\sqrt{\frac{\log(\abs{\Pi}T\delta^{-1})}{T}}\\
      &\leq\bigoh
        (\Vmax+\kappa^{-1/2})\cdot\sqrt{\frac{\coefffull\log(\abs{\Pi}\delta^{-1})\log(T)}{T}}.
\end{align}
Finally, we note that $(\Vmax+\kappa^{-1/2})=\bigoh((\Vmax+\Rmax)e^{2\Rmax})$.

\end{proof}

\subsection{Proofs for SEC Bounds}
\label{sec:main_supporting}

\begin{proof}[\pfref{lem:sec_coverability}]%
  \newcommand{\Ctil}{\wt{C}_{\mathsf{cov}}(\Pi)}%
  \newcommand{\tfrak}{\mathfrak{t}}%
  \newcommand{\Val}{\mathsf{Val}}%
  This proof is based on Proposition 19 of \citet{xie2023role},
  with some additional modifications to handle the preference-based setting.
  Let $T\in\bbN$ and policies $\pi\ind{1},\ldots,\pi\ind{T}$ be
  given, and recall that $\pitil\ind{t}=\samp(\pi\ind{1},\ldots,\pi\ind{t})$.
  Define
  \[
    \delta\ind{t}(\tau,\tautil)
    =       \beta\log\frac{\pi^\iter{t}(\tau)}{\piref(\tau)} - r(\tau) -
      \beta\log\frac{\pi^\iter{t}(\taut)}{\piref(\taut)} + r(\taut),
    \]
    and note that by \cref{lem:implicit_q}, we have
    $\abs*{\delta\ind{t}(\tau,\tautil)}\leq{}4\Vmax$ whenever
    $\tau$ and $\tautil$ share the same initial state $s_1$. Let
    $\En_{\pi,\pi'}$ denote the expectation over trajectories induced
    by sampling $s_1\sim\rho$, $\tau\sim{}\pi\mid{}s_1$, and
    $\tautil\sim\pi'\mid{}s_1$. Meanwhile, let $\En_{\bmu\ind{t}}$
    denote the expectation over trajectories induced by sampling
    $s_1\sim\rho$ and $(\tau,\tautil)\sim\bmu\ind{t}\mid{}s_1$.
    Then our goal is to bound
    \begin{align*}
      \Val\ldef{}\sum_{t=1}^{T}\frac{\prn*{\E_{\pi\ind{t},\pisamp\ind{t-1}} \brk*{\delta\ind{t}(\tau,\tautil)}}^2}{\Vmax^2\vee
(t-1)\cdot\E_{\bmu\ind{t}}\brk*{(\delta\ind{t}(\tau,\tautil))^2} }.
    \end{align*}
    Let
    \[
\nu = \argmin_{\nu\in\Delta((\cS\times\cA)^{H})}\sup_{\tau\in(\cS\times\cA)^{H}}\sup_{\pi\in\Pi}\frac{d^{\pi}(\tau)}{\nu(\tau)}
\]
be the distribution that achieves the value of the coverability
coefficient in \cref{def:coverability}. Let us abbreviate
$\Ccov\equiv\Ccov(\Pi)$. For a trajectory $\tau$, let
\[
  \tfrak(\tau)\ldef{}\min\crl*{t\mid{}\sum_{i<t}d^{\pi\ind{i}}(\tau)\geq\Ccov\cdot\nu(\tau)}.
\]
Then we can bound
\begin{align*}
  \Val\leq
  \underbrace{\sum_{t=1}^{T}\frac{\prn*{\E_{\pi\ind{t},\pisamp\ind{t-1}} \brk*{\delta\ind{t}(\tau,\tautil)\indic\crl{t<\tfrak(\tau)}}}^2}{\Vmax^2\vee
(t-1)\cdot\E_{\bmu\ind{t}}\brk*{(\delta\ind{t}(\tau,\tautil))^2}
  }}_{\rdef{}\I}
  + \underbrace{\sum_{t=1}^{T}\frac{\prn*{\E_{\pi\ind{t},\pitil\ind{t-1}} \brk*{\delta\ind{t}(\tau,\tautil)\indic\crl{t\geq{}\tfrak(\tau)}}}^2}{\Vmax^2\vee
(t-1)\cdot\E_{\bmu\ind{t}}\brk*{(\delta\ind{t}(\tau,\tautil))^2} }}_{\rdef{}\II}.
\end{align*}
We begin by bounding the first term by 
\begin{align*}
  \I\leq{}
  \frac{1}{\Vmax^2}\sum_{t=1}^{T}\prn*{\E_{\pi\ind{t},\pitil\ind{t-1}}
      \brk*{\delta\ind{t}(\tau,\tautil)\indic\crl{t<\tfrak(\tau)}}}^2
  \leq{}
    16\sum_{t=1}^{T}\E_{\pi\ind{t}} \brk*{\indic\crl{t<\tfrak(\tau)}}.
\end{align*}
Letting $\cT\ldef{}(\cS\times\cA)^{H}$, we can further bound this by
\begin{align*}
  \sum_{t=1}^{T}\E_{\pi\ind{t}} \brk*{\indic\crl{t<\tfrak(\tau)}}
  &=
    \sum_{\tau\in\cT}\sum_{t=1}^{T}d^{\pi\ind{t}}(\tau)\indic\crl{t<\tfrak(\tau)}\\
  &=
    \sum_{\tau\in\cT}\prn*{\sum_{i=1}^{\tfrak(\tau)-2}d^{\pi\ind{i}}(\tau)}
    + d^{\pi\ind{\tfrak(\tau)-1}}(\tau)\\
    &\leq{}
      2\Ccov\sum_{\tau\in\cT}\nu(\tau) = 2\Ccov,
\end{align*}
so that $\I\leq{}32\Ccov$.

We now bound term $\II$. Define
$d^{\pi,\pi'}(\tau',\tautil')=\bbP_{s_1\sim\rho,\tau\sim\pi\mid{}s_1,\tautil\sim\pi'\mid{}s_1}(\tau=\tau',\tautil=\tautil')$
and
$d^{\bmu\ind{t}}(\tau',\tautil')=\frac{1}{t-1}\sum_{i<t}d^{\pi\ind{i},\pitil\ind{i}}(\tau',\tautil')$. For
each $t$, we can write
\begin{align*}
  &\E_{\pi\ind{t},\pisamp\ind{t-1}}
  \brk*{\delta\ind{t}(\tau,\tautil)\indic\crl{t<\tfrak(\tau)}}\\
    &=\sum_{\tau,\tautil\in\cT}d^{\pi\ind{t},\pitil\ind{t-1}}(\tau,\tautil)\delta\ind{t}(\tau,\tautil)
      \indic\crl{t\geq\tfrak(\tau)}\\
  &=\sum_{\tau,\tautil\in\cT}d^{\pi\ind{t},\pitil\ind{t-1}}(\tau,\tautil)\delta\ind{t}(\tau,\tautil)
    \prn*{\frac{d^{\bmu\ind{t}}(\tau,\tautil)}{d^{\bmu\ind{t}}(\tau,\tautil)}}^{1/2}\indic\crl{t\geq\tfrak(\tau)}\\
    &\leq\prn*{\sum_{\tau,\tautil\in\cT}\frac{(d^{\pi\ind{t},\pisamp\ind{t-1}}(\tau,\tautil))^2\indic\crl{t\geq\tfrak(\tau)}}{(t-1)\cdot{}d^{\bmu\ind{t}}(\tau,\tautil)}
      }^{1/2}\cdot{}\prn*{(t-1)\cdot{}\En_{\bmu\ind{t}}\brk*{(\delta\ind{t}(\tau,\tautil))^2}}^{1/2},
\end{align*}
where the last inequality is by Cauchy-Schwarz. We conclude that
\begin{align*}
  \II \leq{} \sum_{t=1}^{T}\sum_{\tau,\tautil\in\cT}\frac{(d^{\pi\ind{t},\pisamp\ind{t-1}}(\tau,\tautil))^2\indic\crl{t\geq\tfrak(\tau)}}{(t-1)\cdot{}d^{\bmu\ind{t}}(\tau,\tautil)}.
\end{align*}

To proceed, we restrict our attention to the case where
$\pisamp\ind{t}=\pisamp$ for all $t$ for some fixed $\pisamp$. We
observe that in this case, for all $t$,
\[
  \frac{d^{\pi\ind{t},\pitil\ind{t-1}}(\tau,\tautil)}{d^{\bmu\ind{t}}(\tau,\tautil)}=
  \frac{d^{\pi\ind{t},\pitil}(\tau,\tautil)}{\frac{1}{t-1}\sum_{i<t}d^{\pi\ind{i},\pitil}(\tau,\tautil)}=
  \frac{d^{\pi\ind{t}}(\tau)}{\frac{1}{t-1}\sum_{i<t}d^{\pi\ind{i}}(\tau)},
\]
since $\tau$ and $\tautil$ are conditionally independent given
$s_1$, and since $d^{\pi,\pi'}(\tau,\tautil)=0$ if $\tau,\tautil$ do
not share the same $s_1$. It follows that
  \begin{align*}
  \II &\leq{}
        \sum_{t=1}^{T}\sum_{\tau,\tautil\in\cT}\frac{d^{\pi\ind{t}}(\tau)
        d^{\pi\ind{t},\pisamp}(\tau,\tautil)\indic\crl{t\geq\tfrak(\tau)}}{\sum_{i<t}d^{\pi\ind{i}}(\tau)}\\
      &=
        \sum_{\tau}\sum_{t=1}^{T}\frac{(d^{\pi\ind{t}}(\tau))^2\indic\crl{t\geq\tfrak(\tau)}}{\sum_{i<t}d^{\pi\ind{i}}(\tau)}\\
      &\leq{}
        2\sum_{\tau}\sum_{t=1}^{T}\frac{(d^{\pi\ind{t}}(\tau))^2}{\sum_{i<t}d^{\pi\ind{i}}(\tau)
        + \Ccov{}\nu(\tau)}\\
          &\leq{}
            2\Ccov\sum_{\tau}\nu(\tau)\sum_{t=1}^{T}\frac{d^{\pi\ind{t}}(\tau)}{\sum_{i<t}d^{\pi\ind{i}}(\tau)
        + \Ccov{}\nu(\tau)}.
  \end{align*}
  Finally, by Lemma 4 of \citet{xie2023role}, we have that for all
  $\tau\in\cT$, $\sum_{t=1}^{T}\frac{d^{\pi\ind{t}}(\tau)}{\sum_{i<t}d^{\pi\ind{i}}(\tau)
        + \Ccov{}\nu(\tau)}\leq{}\bigoh(\log(T))$, which yields
      $\II\leq\bigoh(\Ccov\log(T))$. This proves the result.

    \end{proof}

      \begin{proof}[Proof for \cref{ex:linear}]
    We claim for any pair of trajectories $\tau,\tautil$ and function
    $f\in\cF$, we can write
    \begin{align}
      \label{eq:linear}
      \sum_{h=1}^{H}(f(s_h,a_h)-\brk*{\cT_{\beta}f}(s_h,a_h))
      - (f(\stil_h,\atil_h)-\brk*{\cT_{\beta}f}(\stil_h,\atil_h))
      = \tri*{X(\tau,\tautil),W(f)}
    \end{align}
    for embeddings $X(\tau,\tautil),W(f)\in\bbR^{d}$. To see this,
    note that $f(s_h,a_h)=\tri*{\phi(s_h,a_h),\theta_f}$ for some
    $\theta_f\in\bbR^{d}$ with $\nrm*{\theta_f}\leq{}B$ by definition, while the linear MDP property
    implies that we can write
    $\brk*{\cT_{\beta}f}(s_h,a_h)=\tri*{\phi(s_h,a_h),w_f}$ for some
    $w_f\in\bbR^{d}$ with $\nrm*{w_f}\leq\bigoh(\sqrt{d})$. It follows
    that we can take
    \[
      X(\tau,\tautil) = \sum_{h=1}^{H}\phi(s_h,a_h)-\phi(\stil_h,\atil_h) \in\bbR^{d}
    \]
    and
    \[
      W(f) = \theta_f-w_f\in\bbR^{d}.
    \]

With this definition, we observe that in the case where
$\pisamp\ind{t}=\pisamp$ for all $t$, we can write the value of
$\coeff$ for a sequence of policies $\pi\ind{1},\ldots,\pi\ind{T}$ as
\begin{align*}
\sum_{t=1}^{T}\frac{\prn*{\E_{s_1 \sim \rho, \tau \sim
  \pi^\iter{t}\mid s_1, \taut \sim \pisamp \mid s_1}
  \brk*{\tri*{X(\tau,\tautil),W(f\ind{t})}}
  }^2
  }{\Vmax^2\vee
  \sum_{i<t}\E_{s_1 \sim \rho, \tau\sim \pi\ind{i} \mid s_1,\tautil\sim\pisamp\mid{}s_1}   \brk*{\tri*{X(\tau,\tautil),W(f\ind{t})}^2}}
\end{align*}
In particular, if we define $W\ind{t}\ldef{}W(f\ind{t})$ and $X\ind{t}=\E_{s_1 \sim \rho, \tau \sim
  \pi^\iter{t}\mid s_1, \taut \sim \pisamp \mid s_1}
  \brk*{X(\tau,\tautil)}$, it follows from Jensen's inequality that we
  can bound the quantity above by
  \begin{align*}
\sum_{t=1}^{T}\frac{\tri*{X\ind{t},W\ind{t}}^2
  }{\Vmax^2\vee
  \sum_{i<t}\tri*{X\ind{i},W\ind{t}}^2}
\end{align*}
Using that
    $\nrm*{X(\tau,\tautil)},\nrm*{W(f)}\leq\poly(H,d)$, it now follows
    from the standard elliptic potential argument (e.g.,
    \citet{du2021bilinear,jin2021bellman}) that $\coeff(\cF,T;\samp)\leq\bigoht(d)$.
\end{proof}

\section{Additional Proofs}
\label{sec:proofs_additional}

This section contains proofs for supporting results found throughout
\cref{sec:background} and \cref{sec:main}.

\subsection{Proofs from \creftitle{sec:background}}

\begin{proof}[\pfref{prop:dpo_failure}]%
Consider the bandit setting where $H=1$, $\cS=\emptyset$, and
$\cA=\crl{\afrak,\bfrak}$. Let $\beta>0$ be given. We consider the
reward function $r$ given by $r(\afrak)=1$ and
$r(\bfrak)=\frac{1}{2}$. We choose the reference model to set
$\piref(\afrak)=\veps$ and $\piref(\bfrak)=1-\veps$ for a parameter
$\veps\ldef{}\exp(-\frac{c}{\beta})$, where $c>0$ is an absolute
constant whose value will be chosen at the end of the proof. We choose
$\Pi=\crl{\piref,\pistarb}$, which we note satisfies
\cref{ass:realizability} and \cref{ass:vmax} with $\Vmax=\bigoh(1)$.

Specialized to the bandit setting, Online \dpo{} takes the following
simplified form:
\begin{enumerate}
\item Sample pair of actions $a\ind{t},\atil\ind{t}\sim\pi\ind{t}$.
\item Label the actions as $(\ap\ind{t},\am\ind{t})$ according the
  Bradley-Terry model:
  \[
    \bbP(a\ind{t}\psdgt\atil\ind{t}) = \frac{\exp(r(a\ind{t}))}{\exp(r(a\ind{t}))+\exp(r(\atil\ind{t}))},
  \]
  and update $\cD_\pref\ind{t+1}\gets{}\cD_\pref\ind{t}\cup\crl{(\ap\ind{t},\am\ind{t})}$.
\item Compute $\pi\ind{t+1}$ via
  \begin{equation}
    \label{eq:online_dpo_bandit}
    \pi\ind{t+1}=\argmin_{\pi\in\Pi}
\sum_{(\ap,\am) \in \Dcal\ind{t+1}_\pref} - \log\left[\sigma\left( \beta\log\frac{\pi(\ap)}{\piref(\ap)} - \beta\log\frac{\pi(\am)}{\piref(\am)} \right) \right].
\end{equation}
\end{enumerate}
Our construction uses the fact that depending on the preference dataset $\cD_\pref\ind{t}$, the
minimizer in \cref{eq:online_dpo_bandit} may not be uniquely defined.
Let $\cE\ind{t}$ denote the event that at iteration $t$,
$a\ind{t}=\atil\ind{t}=\bfrak$. We appeal to a technical lemma.
\begin{lemma}
  \label{lem:online_dpo_bad}
  Suppose we initialize with $\pi\ind{1}=\piref$.
  As long as $c\leq\frac{1}{8}$, $\veps\leq{}1/2$, the following properties hold:
  \begin{itemize}
  \item $\bbP(\cE\ind{t}\mid{}\cE\ind{1},\ldots\cE\ind{t-1})\geq{}1-2\veps$.
  \item Whenever $\cE\ind{1},\ldots,\cE\ind{t}$ hold, we can choose
    the policy
    $\pi\ind{t+1}$ to satisfy $\pi\ind{t+1}=\piref$, which has
    \[
      \max_{\pi}J_\beta(\pi) - J_\beta(\pi\ind{t+1})  =
      \max_{\pi}J_\beta(\pi) - J_\beta(\piref) \geq{} \frac{1}{8}
      \]
  \end{itemize}
\end{lemma}
By \cref{lem:online_dpo_bad} and the union bound, we have that
\begin{align*}
  \bbP(\cE\ind{1},\ldots,\cE\ind{T}) \geq{} (1-2\veps)^{T} \geq{} e^{-1}
\end{align*}
as long as $T\leq{} \frac{1}{2\veps}$. It follows that whenever this
occurs, $\max_{\pi}J_\beta(\pi) - J_\beta(\pi\ind{t}) \geq{} \frac{1}{8}$ for all $t\in\brk{T+1}$.

Note that since online \dpo selects $\pi\ind{t}=\piref$ for all $t$ in our
counterexample above, this also immediately implies a lower bound for
offline \dpo (interpreting $\pi\ind{T+1}$ as the policy returned by
offline \dpo).

\end{proof}

\begin{proof}[\pfref{lem:online_dpo_bad}]
  We prove this claim inductively. Let $t\in\brk{T}$ be fixed, and
  suppose the claim holds for $1,\ldots,t-1$. If we assume
  $\cE\ind{1},\ldots,\cE\ind{t-1}$ hold, then we have
  $\pi\ind{t}=\piref$ inductively. In this case,
  \[
    \bbP(a\ind{t}=\wt{a}\ind{t}=\bfrak) = (\piref(\bfrak))^2=(1-\veps)^2\geq{}1-2\veps,
  \]
  so that
  $\bbP(\cE\ind{t}\mid{}\cE\ind{1},\ldots\cE\ind{t-1})\geq{}1-2\veps$
  as desired.

  Now, for the second part of the claim, suppose that
  $\cE\ind{1},\ldots,\cE\ind{t+1}$ hold. Then for all $t'\in\brk{t+1}$,
  $a_{+}\ind{t'}=a_{-}\ind{t'}=\bfrak$, which implies that
  \[
\sum_{(\ap,\am) \in \Dcal\ind{t+1}_\pref} - \log\left[\sigma\left(
    \beta\log\frac{\pi(\ap)}{\piref(\ap)} -
    \beta\log\frac{\pi(\am)}{\piref(\am)} \right) \right]
=-\log(\sigma(0))\cdot{}t
\]
for all $\pi\in\Pi$ such that $\pi\ll\piref$. It follows that
$\pi\ind{t+1}=\piref$ is a valid
minimizer for \cref{eq:online_dpo_bandit}.

Finally, we compute that as long as $\veps\leq{}1/2$ and $c\leq\frac{1}{8}$
\begin{align*}
  \max_{\pi}J_\beta(\pi) - J_\beta(\piref)
  &\geq{} \max_{\pi}J(\pi) - J(\piref) -\beta\log(\veps^{-1})\\
  &= (1 -
    (1-\veps)\cdot{}\tfrac{1}{2}-\veps{}\cdot1)-\beta\log(\veps^{-1})
  \geq{} \frac{1}{4}-c\geq{}\frac{1}{8}.
\end{align*}

\end{proof}

\section{Experiments: Additional Results and Details}
\label{sec:exp_details}

\cref{tb:openllm} displays the performance of \algshort and the comparator models described in \cref{sec:experiments} on additional reasoning tasks. We observe that the \algshort model outperforms Llama-3-8B-it, and still comparable to \llamafinal, which uses 4x more data than \algshort. The average over all academic benchmarks is 59.94 for \llamafinal vs.~59.61 for \algshort-iter3, in addition to the performance gain from \algshort-iter3 in the chat benchmarks (\cref{tb:llama_main}).

\begin{table}[th]
\centering
\begin{adjustbox}{max width=\textwidth}
\begin{tabular}{c|cccc}
\textbf{Model}             & \textbf{ARC-C} & \textbf{TruthfulQA} & \textbf{HellaSwag} & \textbf{WinoGrande} \\ \hline
\llamasft   & 56.23                      & 53.43                     & 78.66                          & 76.64               \\ \hline
\dpo-iter1             & 57.08                      & 59.73                     & 79.97                          & 76.24               \\
\dpo-iter2             & 56.57                      & 60.66            & 79.74                          & 76.16               \\
\dpo-iter3             & 57                      & \textbf{61.79}                      & 79.65                          & \textbf{76.95}      \\ \hline
\algshort-iter1             & \textbf{57.34}             & 59.28                     & \textbf{80.01}                 & 76.4                \\
\algshort-iter2             & 57.25                      & 59.72                     & 79.85                          & 76.4                \\
\algshort-iter3             & 56.23                      & 59.74                     & 79.16                          & 76.64               \\ \hline\hline
\llamafinal & 54.35                      & \underline{62.23}               & \underline{80.84}                    & \underline{77.19}         \\
Llama-3-8B-it         & 56.83                      & 51.65                     & 75.75                          & 71.74        
\end{tabular}
\end{adjustbox}
\vspace{3pt}
\caption{\algshort and comparator models on additional reasoning benchmarks. \textbf{Bold} results are the best ones with the same data usage. \underline{Underlined} results are superior to \algshort, but either from a model requiring more data or industry-level.}
\label{tb:openllm}
\end{table}

\paragraph{Implementation details}
The experiments were conducted on 8 x Nvidia H100 GPUs.
In our implementation, we mainly follow the general version of \algshort (\cref{alg:general}), and we pick $\pisamp^\iter{t} = \pi^\iter{t}$ and $\cDopt\ind{t} =  \Dcal_\pref^\iter{t}$.
In each iteration, we fix the base model (\llamasft) as $\piref$, set $\beta = 0.1$, use a global batch size of $16$, and use a learning rate of $5\times 10^{-7}$ with cosine scheduling.
The $\alpha$ parameter follows the schedule $\{1\times 10^{-5},5\times 10^{-6},0\}$ for the three iterations. We clip the $\log\frac{\pi(\tau)}{\piref(\tau)}$ term for both positive and negative trajectories to $[-500,500]$, but only for the exploration term, in order to enhance stability. This is motivated by \cref{ass:vmax}. The number of training epochs for each iteration is 2, and the warmup ratio is 0.03.

\newpage

\neurips{\section*{NeurIPS Paper Checklist}

\begin{enumerate}

\item {\bf Claims}
    \item[] Question: Do the main claims made in the abstract and introduction accurately reflect the paper's contributions and scope?
    \item[] Answer: \answerYes{} %
    \item[] Justification: %
    \item[] Guidelines:
    \begin{itemize}
        \item The answer NA means that the abstract and introduction do not include the claims made in the paper.
        \item The abstract and/or introduction should clearly state the claims made, including the contributions made in the paper and important assumptions and limitations. A No or NA answer to this question will not be perceived well by the reviewers. 
        \item The claims made should match theoretical and experimental results, and reflect how much the results can be expected to generalize to other settings. 
        \item It is fine to include aspirational goals as motivation as long as it is clear that these goals are not attained by the paper. 
    \end{itemize}

\item {\bf Limitations}
    \item[] Question: Does the paper discuss the limitations of the work performed by the authors?
    \item[] Answer: \answerYes{} %
    \item[] Justification: %
    \item[] Guidelines:
    \begin{itemize}
        \item The answer NA means that the paper has no limitation while the answer No means that the paper has limitations, but those are not discussed in the paper. 
        \item The authors are encouraged to create a separate "Limitations" section in their paper.
        \item The paper should point out any strong assumptions and how robust the results are to violations of these assumptions (e.g., independence assumptions, noiseless settings, model well-specification, asymptotic approximations only holding locally). The authors should reflect on how these assumptions might be violated in practice and what the implications would be.
        \item The authors should reflect on the scope of the claims made, e.g., if the approach was only tested on a few datasets or with a few runs. In general, empirical results often depend on implicit assumptions, which should be articulated.
        \item The authors should reflect on the factors that influence the performance of the approach. For example, a facial recognition algorithm may perform poorly when image resolution is low or images are taken in low lighting. Or a speech-to-text system might not be used reliably to provide closed captions for online lectures because it fails to handle technical jargon.
        \item The authors should discuss the computational efficiency of the proposed algorithms and how they scale with dataset size.
        \item If applicable, the authors should discuss possible limitations of their approach to address problems of privacy and fairness.
        \item While the authors might fear that complete honesty about limitations might be used by reviewers as grounds for rejection, a worse outcome might be that reviewers discover limitations that aren't acknowledged in the paper. The authors should use their best judgment and recognize that individual actions in favor of transparency play an important role in developing norms that preserve the integrity of the community. Reviewers will be specifically instructed to not penalize honesty concerning limitations.
    \end{itemize}

\item {\bf Theory Assumptions and Proofs}
    \item[] Question: For each theoretical result, does the paper provide the full set of assumptions and a complete (and correct) proof?
    \item[] Answer: \answerYes{} %
    \item[] Justification: %
    \item[] Guidelines:
    \begin{itemize}
        \item The answer NA means that the paper does not include theoretical results. 
        \item All the theorems, formulas, and proofs in the paper should be numbered and cross-referenced.
        \item All assumptions should be clearly stated or referenced in the statement of any theorems.
        \item The proofs can either appear in the main paper or the supplemental material, but if they appear in the supplemental material, the authors are encouraged to provide a short proof sketch to provide intuition. 
        \item Inversely, any informal proof provided in the core of the paper should be complemented by formal proofs provided in appendix or supplemental material.
        \item Theorems and Lemmas that the proof relies upon should be properly referenced. 
    \end{itemize}

    \item {\bf Experimental Result Reproducibility}
    \item[] Question: Does the paper fully disclose all the information needed to reproduce the main experimental results of the paper to the extent that it affects the main claims and/or conclusions of the paper (regardless of whether the code and data are provided or not)?
    \item[] Answer: \answerYes{} %
    \item[] Justification: %
    \item[] Guidelines:
    \begin{itemize}
        \item The answer NA means that the paper does not include experiments.
        \item If the paper includes experiments, a No answer to this question will not be perceived well by the reviewers: Making the paper reproducible is important, regardless of whether the code and data are provided or not.
        \item If the contribution is a dataset and/or model, the authors should describe the steps taken to make their results reproducible or verifiable. 
        \item Depending on the contribution, reproducibility can be accomplished in various ways. For example, if the contribution is a novel architecture, describing the architecture fully might suffice, or if the contribution is a specific model and empirical evaluation, it may be necessary to either make it possible for others to replicate the model with the same dataset, or provide access to the model. In general. releasing code and data is often one good way to accomplish this, but reproducibility can also be provided via detailed instructions for how to replicate the results, access to a hosted model (e.g., in the case of a large language model), releasing of a model checkpoint, or other means that are appropriate to the research performed.
        \item While NeurIPS does not require releasing code, the conference does require all submissions to provide some reasonable avenue for reproducibility, which may depend on the nature of the contribution. For example
        \begin{enumerate}
            \item If the contribution is primarily a new algorithm, the paper should make it clear how to reproduce that algorithm.
            \item If the contribution is primarily a new model architecture, the paper should describe the architecture clearly and fully.
            \item If the contribution is a new model (e.g., a large language model), then there should either be a way to access this model for reproducing the results or a way to reproduce the model (e.g., with an open-source dataset or instructions for how to construct the dataset).
            \item We recognize that reproducibility may be tricky in some cases, in which case authors are welcome to describe the particular way they provide for reproducibility. In the case of closed-source models, it may be that access to the model is limited in some way (e.g., to registered users), but it should be possible for other researchers to have some path to reproducing or verifying the results.
        \end{enumerate}
    \end{itemize}

\item {\bf Open access to data and code}
    \item[] Question: Does the paper provide open access to the data and code, with sufficient instructions to faithfully reproduce the main experimental results, as described in supplemental material?
    \item[] Answer: \answerNo{} %
    \item[] Justification: This is a primarily theoretical work.
    \item[] Guidelines:
    \begin{itemize}
        \item The answer NA means that paper does not include experiments requiring code.
        \item Please see the NeurIPS code and data submission guidelines (\url{https://nips.cc/public/guides/CodeSubmissionPolicy}) for more details.
        \item While we encourage the release of code and data, we understand that this might not be possible, so “No” is an acceptable answer. Papers cannot be rejected simply for not including code, unless this is central to the contribution (e.g., for a new open-source benchmark).
        \item The instructions should contain the exact command and environment needed to run to reproduce the results. See the NeurIPS code and data submission guidelines (\url{https://nips.cc/public/guides/CodeSubmissionPolicy}) for more details.
        \item The authors should provide instructions on data access and preparation, including how to access the raw data, preprocessed data, intermediate data, and generated data, etc.
        \item The authors should provide scripts to reproduce all experimental results for the new proposed method and baselines. If only a subset of experiments are reproducible, they should state which ones are omitted from the script and why.
        \item At submission time, to preserve anonymity, the authors should release anonymized versions (if applicable).
        \item Providing as much information as possible in supplemental material (appended to the paper) is recommended, but including URLs to data and code is permitted.
    \end{itemize}

\item {\bf Experimental Setting/Details}
    \item[] Question: Does the paper specify all the training and test details (e.g., data splits, hyperparameters, how they were chosen, type of optimizer, etc.) necessary to understand the results?
    \item[] Answer: \answerYes{} %
    \item[] Justification: %
    \item[] Guidelines:
    \begin{itemize}
        \item The answer NA means that the paper does not include experiments.
        \item The experimental setting should be presented in the core of the paper to a level of detail that is necessary to appreciate the results and make sense of them.
        \item The full details can be provided either with the code, in appendix, or as supplemental material.
    \end{itemize}

\item {\bf Experiment Statistical Significance}
    \item[] Question: Does the paper report error bars suitably and correctly defined or other appropriate information about the statistical significance of the experiments?
    \item[] Answer: \answerYes{} %
    \item[] Justification: %
    \item[] Guidelines:
    \begin{itemize}
        \item The answer NA means that the paper does not include experiments.
        \item The authors should answer "Yes" if the results are accompanied by error bars, confidence intervals, or statistical significance tests, at least for the experiments that support the main claims of the paper.
        \item The factors of variability that the error bars are capturing should be clearly stated (for example, train/test split, initialization, random drawing of some parameter, or overall run with given experimental conditions).
        \item The method for calculating the error bars should be explained (closed form formula, call to a library function, bootstrap, etc.)
        \item The assumptions made should be given (e.g., Normally distributed errors).
        \item It should be clear whether the error bar is the standard deviation or the standard error of the mean.
        \item It is OK to report 1-sigma error bars, but one should state it. The authors should preferably report a 2-sigma error bar than state that they have a 96\% CI, if the hypothesis of Normality of errors is not verified.
        \item For asymmetric distributions, the authors should be careful not to show in tables or figures symmetric error bars that would yield results that are out of range (e.g. negative error rates).
        \item If error bars are reported in tables or plots, The authors should explain in the text how they were calculated and reference the corresponding figures or tables in the text.
    \end{itemize}

\item {\bf Experiments Compute Resources}
    \item[] Question: For each experiment, does the paper provide sufficient information on the computer resources (type of compute workers, memory, time of execution) needed to reproduce the experiments?
    \item[] Answer: \answerYes{} %
    \item[] Justification: %
    \item[] Guidelines:
    \begin{itemize}
        \item The answer NA means that the paper does not include experiments.
        \item The paper should indicate the type of compute workers CPU or GPU, internal cluster, or cloud provider, including relevant memory and storage.
        \item The paper should provide the amount of compute required for each of the individual experimental runs as well as estimate the total compute. 
        \item The paper should disclose whether the full research project required more compute than the experiments reported in the paper (e.g., preliminary or failed experiments that didn't make it into the paper). 
    \end{itemize}
    
\item {\bf Code Of Ethics}
    \item[] Question: Does the research conducted in the paper conform, in every respect, with the NeurIPS Code of Ethics \url{https://neurips.cc/public/EthicsGuidelines}?
    \item[] Answer: \answerYes{} %
    \item[] Justification: %
    \item[] Guidelines:
    \begin{itemize}
        \item The answer NA means that the authors have not reviewed the NeurIPS Code of Ethics.
        \item If the authors answer No, they should explain the special circumstances that require a deviation from the Code of Ethics.
        \item The authors should make sure to preserve anonymity (e.g., if there is a special consideration due to laws or regulations in their jurisdiction).
    \end{itemize}

\item {\bf Broader Impacts}
    \item[] Question: Does the paper discuss both potential positive societal impacts and negative societal impacts of the work performed?
    \item[] Answer: \answerNA{} %
    \item[] Justification: This is a primarily theoretical work.
    \item[] Guidelines:
    \begin{itemize}
        \item The answer NA means that there is no societal impact of the work performed.
        \item If the authors answer NA or No, they should explain why their work has no societal impact or why the paper does not address societal impact.
        \item Examples of negative societal impacts include potential malicious or unintended uses (e.g., disinformation, generating fake profiles, surveillance), fairness considerations (e.g., deployment of technologies that could make decisions that unfairly impact specific groups), privacy considerations, and security considerations.
        \item The conference expects that many papers will be foundational research and not tied to particular applications, let alone deployments. However, if there is a direct path to any negative applications, the authors should point it out. For example, it is legitimate to point out that an improvement in the quality of generative models could be used to generate deepfakes for disinformation. On the other hand, it is not needed to point out that a generic algorithm for optimizing neural networks could enable people to train models that generate Deepfakes faster.
        \item The authors should consider possible harms that could arise when the technology is being used as intended and functioning correctly, harms that could arise when the technology is being used as intended but gives incorrect results, and harms following from (intentional or unintentional) misuse of the technology.
        \item If there are negative societal impacts, the authors could also discuss possible mitigation strategies (e.g., gated release of models, providing defenses in addition to attacks, mechanisms for monitoring misuse, mechanisms to monitor how a system learns from feedback over time, improving the efficiency and accessibility of ML).
    \end{itemize}
    
\item {\bf Safeguards}
    \item[] Question: Does the paper describe safeguards that have been put in place for responsible release of data or models that have a high risk for misuse (e.g., pretrained language models, image generators, or scraped datasets)?
    \item[] Answer: \answerNA{} %
    \item[] Justification: %
    \item[] Guidelines:
    \begin{itemize}
        \item The answer NA means that the paper poses no such risks.
        \item Released models that have a high risk for misuse or dual-use should be released with necessary safeguards to allow for controlled use of the model, for example by requiring that users adhere to usage guidelines or restrictions to access the model or implementing safety filters. 
        \item Datasets that have been scraped from the Internet could pose safety risks. The authors should describe how they avoided releasing unsafe images.
        \item We recognize that providing effective safeguards is challenging, and many papers do not require this, but we encourage authors to take this into account and make a best faith effort.
    \end{itemize}

\item {\bf Licenses for existing assets}
    \item[] Question: Are the creators or original owners of assets (e.g., code, data, models), used in the paper, properly credited and are the license and terms of use explicitly mentioned and properly respected?
    \item[] Answer: \answerYes{} %
    \item[] Justification: %
    \item[] Guidelines:
    \begin{itemize}
        \item The answer NA means that the paper does not use existing assets.
        \item The authors should cite the original paper that produced the code package or dataset.
        \item The authors should state which version of the asset is used and, if possible, include a URL.
        \item The name of the license (e.g., CC-BY 4.0) should be included for each asset.
        \item For scraped data from a particular source (e.g., website), the copyright and terms of service of that source should be provided.
        \item If assets are released, the license, copyright information, and terms of use in the package should be provided. For popular datasets, \url{paperswithcode.com/datasets} has curated licenses for some datasets. Their licensing guide can help determine the license of a dataset.
        \item For existing datasets that are re-packaged, both the original license and the license of the derived asset (if it has changed) should be provided.
        \item If this information is not available online, the authors are encouraged to reach out to the asset's creators.
    \end{itemize}

\item {\bf New Assets}
    \item[] Question: Are new assets introduced in the paper well documented and is the documentation provided alongside the assets?
    \item[] Answer: \answerNA{} %
    \item[] Justification: %
    \item[] Guidelines:
    \begin{itemize}
        \item The answer NA means that the paper does not release new assets.
        \item Researchers should communicate the details of the dataset/code/model as part of their submissions via structured templates. This includes details about training, license, limitations, etc. 
        \item The paper should discuss whether and how consent was obtained from people whose asset is used.
        \item At submission time, remember to anonymize your assets (if applicable). You can either create an anonymized URL or include an anonymized zip file.
    \end{itemize}

\item {\bf Crowdsourcing and Research with Human Subjects}
    \item[] Question: For crowdsourcing experiments and research with human subjects, does the paper include the full text of instructions given to participants and screenshots, if applicable, as well as details about compensation (if any)? 
    \item[] Answer: \answerNA{} %
    \item[] Justification: %
    \item[] Guidelines:
    \begin{itemize}
        \item The answer NA means that the paper does not involve crowdsourcing nor research with human subjects.
        \item Including this information in the supplemental material is fine, but if the main contribution of the paper involves human subjects, then as much detail as possible should be included in the main paper. 
        \item According to the NeurIPS Code of Ethics, workers involved in data collection, curation, or other labor should be paid at least the minimum wage in the country of the data collector. 
    \end{itemize}

\item {\bf Institutional Review Board (IRB) Approvals or Equivalent for Research with Human Subjects}
    \item[] Question: Does the paper describe potential risks incurred by study participants, whether such risks were disclosed to the subjects, and whether Institutional Review Board (IRB) approvals (or an equivalent approval/review based on the requirements of your country or institution) were obtained?
    \item[] Answer: \answerNA{} %
    \item[] Justification: %
    \item[] Guidelines:
    \begin{itemize}
        \item The answer NA means that the paper does not involve crowdsourcing nor research with human subjects.
        \item Depending on the country in which research is conducted, IRB approval (or equivalent) may be required for any human subjects research. If you obtained IRB approval, you should clearly state this in the paper. 
        \item We recognize that the procedures for this may vary significantly between institutions and locations, and we expect authors to adhere to the NeurIPS Code of Ethics and the guidelines for their institution. 
        \item For initial submissions, do not include any information that would break anonymity (if applicable), such as the institution conducting the review.
    \end{itemize}

\end{enumerate}

}

\end{document}